\documentclass{article}

\usepackage{microtype}
\usepackage{graphicx}
\usepackage{booktabs} 
\usepackage{float}

\usepackage{hyperref}

\usepackage{graphicx}
\usepackage{subcaption}
\usepackage{threeparttable}

\usepackage{bm}
\usepackage{amsmath}
\usepackage{amssymb}
\usepackage{amsfonts}
\usepackage{amsthm}
\usepackage{multirow}
\usepackage{adjustbox}
\usepackage{caption}
\usepackage{subcaption}
\usepackage{mathtools}
\usepackage{enumitem}

\newtheorem{definition}{Definition}

\newtheorem{lemma}{Lemma}

\newcommand{\E}{\mathbb{E}}

\def \xx {{\bm{x}}}

\def \pp {{\bm{p}}}
\def \qq {{\bm{q}}}

\def \zz {{\bm{z}}}

\def \btheta {{\bm{\theta}}}
\def \X  {\mathcal{X}}
\def \Y  {\mathcal{Y}}

\def \L  {\mathcal{L}}

\def \D  {\mathcal{D}}

\DeclareMathOperator*{\argmin}{arg\,min}

\usepackage[accepted]{icml2020}

\icmltitlerunning{Normalized Loss Functions for Deep Learning with Noisy Labels}

\begin{document}

\twocolumn[
\icmltitle{Normalized Loss Functions for Deep Learning with Noisy Labels}



\icmlsetsymbol{equal}{*}

\begin{icmlauthorlist}
\icmlauthor{Xingjun Ma}{equal,uom}
\icmlauthor{Hanxun Huang}{equal,uom}
\icmlauthor{Yisen Wang}{sjtu}
\icmlauthor{Simone Romano}{}
\icmlauthor{Sarah Erfani}{uom}
\icmlauthor{James Bailey}{uom}
\end{icmlauthorlist}

\icmlaffiliation{uom}{The University of Melbourne, Australia}
\icmlaffiliation{sjtu}{Shanghai Jiao Tong University, China}

\icmlcorrespondingauthor{Yisen Wang}{eewangyisen@gmail.com}

\icmlkeywords{Deep Learning, Robust Loss, Noisy Labels}

\vskip 0.3in
]



\printAffiliationsAndNotice{\icmlEqualContribution} 

\begin{abstract}
Robust loss functions are essential for training accurate deep neural networks (DNNs) in the presence of noisy (incorrect) labels. It has been shown that the commonly used Cross Entropy (CE) loss is not robust to noisy labels. Whilst new loss functions have been designed, they are only partially robust. In this paper, we theoretically show by applying a simple normalization that: \emph{any loss can be made robust to noisy labels}. However, in practice, simply being robust is not sufficient for a loss function to train accurate DNNs. By investigating several robust loss functions, we find that they suffer from a problem of {\em underfitting}. To address this, we propose a framework to build robust loss functions called \emph{Active Passive Loss} (APL). APL combines two robust loss functions that mutually boost each other. Experiments on benchmark datasets demonstrate that the family of new loss functions created by our APL framework can consistently outperform state-of-the-art methods by large margins, especially under large noise rates such as 60\% or 80\% incorrect labels.
\end{abstract}

\section{Introduction}
\label{sec:intro}
Training accurate deep neural networks (DNNs) in the presence of noisy (incorrect) labels is of great practical importance. Different approaches have been proposed for robust learning with noisy labels. This includes 1) label correction methods that aim to identify and correct wrong labels \cite{xiao2015learning,vahdat2017toward,veit2017learning,li2017learning}; 2) loss correction methods that correct the loss function based on an estimated noise transition matrix \cite{sukhbaatar2014training,reed2014training,patrini2017making,han2018masking}; 3) refined training strategies that modify the training procedure to be more adaptive to incorrect labels \cite{jiang2018mentornet,wang2018iterative,tanaka2018joint,ma2018dimensionality,han2018co}; and 4) robust loss functions that are inherently tolerant to noisy labels \cite{ghosh2017robust,zhang2018generalized,wang2019symmetric}. Compared to the first three approaches that may suffer from inaccurate noise estimation or involve sophisticated training procedure modifications, robust loss functions provide a simpler solution, which is also the main focus of this paper.

It has been theoretically shown that some loss functions such as Mean Absolute Error (MAE) are robust to label noise, while others are not, which unfortunately includes the commonly used Cross Entropy (CE) loss. 
This has motivated a body of work to design new loss functions that are inherently robust to noisy labels.
For example, Generalized Cross Entropy (GCE) \cite{zhang2018generalized} was proposed to improve the robustness of CE against noisy labels. GCE can be seen as a generalized mixture of CE and MAE, and is only robust when reduced to the MAE loss. Recently, a Symmetric Cross Entropy (SCE) \cite{wang2019symmetric} loss was suggested as a robustly boosted version of CE. SCE combines the CE loss with a Reverse Cross Entropy (RCE) loss, and only the RCE term is robust.
Whilst these loss functions have demonstrated improved robustness, theoretically, they are only partially robust to noisy labels.

Different from previous works, in this paper, 
we theoretically show that any loss can be made robust to noisy labels, and all is needed is a simple normalization.
However, in practice, simply being robust is not enough for a loss function to train accurate DNNs. By investigating several robust loss functions, we find that they all suffer from an underfitting problem. Inspired by recent developments in this field, we propose to characterize existing loss functions into two types: 1) ``Active" loss, which only explicitly maximizes the probability of being in the labeled class, and 2) ``Passive" loss, which also explicitly minimizes the probabilities of being in other classes. Based on this characterization, we further propose a novel framework to build a new set of robust loss functions called \emph{Active Passive Losses} (APLs).
We show that under this framework, existing loss functions can be reworked to achieve the state-of-the-art for training DNNs with noisy labels. Our key contributions are:
\begin{itemize}[leftmargin= *]
  \item  We provide new theoretical insights into robust loss functions demonstrating that a simple normalization can make any loss function robust to noisy labels.
  
  \item We identify that existing robust loss functions suffer from an underfitting problem. To address this, we propose a generic framework \emph{Active Passive Loss} (APL) to build new loss functions with theoretically guaranteed robustness and sufficient learning properties.
 
  \item We empirically demonstrate that the family of new loss functions created following our \emph{APL} framework can outperform the state-of-the-art methods by considerable margins, especially under large noise rates of 60\% or 80\%.
\end{itemize}

\section{Related Work}
\label{sec:related}
We briefly review existing approaches for robust learning with noisy labels.

\noindent\textbf{1) Label correction methods.}
The idea of label correction is to improve the quality of the raw labels, possibly correcting wrong labels into correct ones. One common approach is to apply corrections via a clean label inference step using complex noise models characterized by directed graphical models \cite{xiao2015learning}, conditional random fields \cite{vahdat2017toward}, neural networks \cite{lee2017cleannet,veit2017learning} or knowledge graphs \cite{li2017learning}. These methods require support from extra clean data or a potentially expensive detection process to estimate the noise model. 

\noindent\textbf{2) Loss correction methods.}
This approach improves robustness by modifying the loss function during training, based on label-dependent weights \cite{natarajan2013learning} or an estimated noise transition matrix that defines the probability of mislabeling one class with another \cite{han2018masking}. Backward and Forward \cite{patrini2017making} are two noise transition matrix based loss correction methods.
Work in \cite{goldberger2016training,sukhbaatar2014training} augments the correction architecture by adding a linear layer on top of the neural network. Bootstrap \cite{reed2014training}  uses a combination of raw labels and their predicted labels. Label Smoothing Regularization (LSR) \cite{szegedy2016rethinking,pereyra2017regularizing} uses soft labels in place of one-hot labels to alleviate overfitting to noisy labels. Loss correction methods are sensitive to the noise transition matrix. Given that ground-truth is not always available, this matrix is typically difficult to estimate.

\noindent\textbf{3) Refined training strategies.}
This direction designs adaptive training strategies that are more robust to noisy labels. MentorNet \cite{jiang2018mentornet,yu2019does} supervises the training of a StudentNet by a learned sample weighting scheme in favor of probably correct labels. 
SeCoST extends MentorNet to a cascade of student-teacher pairs via a knowledge transfer method \cite{kumar2019secost}.
Decoupling training strategy \cite{malach2017decoupling} trains two networks simultaneously, and parameters are updated when their predictions disagree. 
Co-teaching \cite{han2018co} allows one network learn from the other network's most confident samples. These studies all require an auxiliary network for sample weighting or learning supervision. D2L \cite{ma2018dimensionality} uses subspace dimensionality adapted labels for learning, paired with a training process monitor. 
The joint optimization framework~\cite{tanaka2018joint} updates DNN parameters and labels alternately.
\citet{kim2019nlnl} use complementary labels to mitigate overfitting to original labels. \citet{xu2019l_dmi} introduce a Determinant-based Mutual Information (DMI) loss for robust fine-tuning of a CE pre-trained model. These methods either rely on complex interventions into the learning process which are hard to adapt and tune, or are sensitive to hyperparameters like training epochs and learning rate.

\noindent\textbf{4) Robust loss functions.} 
Compared to the above three types of methods, robust loss functions are a simpler and arguably more generic solution for robust learning.
Previous work has theoretically proved that some loss functions such as Mean Absolute Error (MAE) are robust to noisy labels, while others like the commonly used Cross Entropy (CE) loss are not \cite{ghosh2017robust}.
However, training with MAE has been found very challenging due to slow convergence caused by gradient saturation \cite{zhang2018generalized}.
The Generalized Cross Entropy (GCE) loss \cite{zhang2018generalized} applies a Box-Cox transformation to probabilities (power law function of probability with exponent $\rho \in (0, 1]$) which can behave like a generalized mixture of MAE and CE. Recently, \citet{wang2019symmetric} proposed the Symmetric Cross Entropy (SCE) which combines a Reverse Cross Entropy (RCE) together with the CE loss.  Both GCE and SCE are only partially robust to noisy labels. For example, GCE is only robust when it reduces to the MAE loss with $\rho=1$. For SCE, only its RCE term is robust. Empirically (rather than theoretically) justified approaches that directly modify the magnitude of the loss gradients are also an active line of research \cite{wang2019imae, wang2019derivative}. 

In this paper, we theoretically prove that, with simple normalization, any loss can be made robust to noisy labels. This new theoretical insight can serve as a basic principle for designing new robust loss functions. It also can reshape the design of new loss functions towards other properties rather than robustness.

\section{Any Loss can be Robust to Noisy Labels}\label{sec:any_loss_robust}
We next introduce some background knowledge about robust classification with noisy labels, then propose a simple but theoretically sound normalization method that can be applied to any loss function to make it robust to noisy labels.

\subsection{Preliminaries}\label{sec:preliminary}
Given a $K$-class dataset with noisy labels as $\mathcal{D} = \{(\xx, y)^{(i)}\}_{i=1}^n$, with $\xx \in \X \subset \mathbb{R}^d$ denoting a sample and $y \in \Y = \{1, \cdots, K\}$ its annotated label (possibly incorrect).
We denote the distribution over different labels for sample $\xx$ by $\qq(k|\xx)$, and $\sum_{k=1}^{K}\qq(k|\xx)=1$. In this paper, we focus on the common case where there is only one single label $y$ for $\xx$: i.e. $\qq(y|\xx)=1$ and $\qq(k \neq y|\xx)=0$. In this case, $\qq$ is simply the one-hot encoding of the label.

We denote the true label of $\xx$ as $y^{*}$.
While noisy labels may arise in different ways, one common assumption is that, given the true labels, the noise is conditionally independent to the inputs, i.e., $\qq(y=k|y^{*}=j, \xx) = \qq(y=k|y^{*}=j)$.
Under this assumption, label noise can be either \emph{symmetric} (or uniform), or \emph{asymmetric} (or class-conditional). We denote the overall noise rate by $\eta \in[0, 1]$ and the class-wise noise rate from class $j$ to class $k$ by $\eta_{jk}$. Then, for symmetric noise, $\eta_{jk} = \frac{\eta}{K - 1}$ for $j \neq k$ and $\eta_{jk} = 1 - \eta$ for $j = k$.  For asymmetric noise, $\eta_{jk}$ is conditioned on both the true class $j$ and mislabeled class $k$. 

Classification is to learn a function $f: \X \rightarrow \Y$ (as represented by a DNN) that maps the input space to the label space.
For a sample $\xx$, we denote the probability output of a DNN classifier $f(\xx)$ as: $\pp(k|\xx) = \frac{e^{\zz_{k}}}{\sum_{j=1}^K e^{\zz_{j}}}$, where $\zz_k$ denotes the logits output of the network with respect to class $k$.
Training classifier $f$ is to find a set of optimal parameters $\btheta$ that minimize the empirical risk defined by a loss function: $\btheta := \argmin_{\btheta} \; \sum_{i=1}^{n} \L(f(\xx_i), y_i)$, where $\L(f(\xx), y)$ is the loss of $f$ with respect to label $y$.
Next, we briefly introduce four loss functions that are either popularly used or recently proposed for robust classification with noisy labels. 

\noindent\textbf{Existing loss functions.} 
The commonly used Cross Entropy (CE) loss on sample $\xx$ is defined as: $CE = -\sum_{k=1}^{K} \qq(k|\xx) \log \pp(k|\xx)$, which has been proved not robust to noisy labels \cite{ghosh2017robust}.

Mean Absolute Error (MAE) is also a popular classification loss, and is defined as: $MAE = \sum_{k=1}^{K} |\pp(k|\xx) - \qq(k|\xx)|$. 
MAE is provably robust to label noise \cite{ghosh2017robust}.

The recently proposed Reverse Cross Entropy (RCE) loss  \cite{wang2019symmetric} is defined as: $RCE = -\sum_{k=1}^{K} \pp(k|\xx) \log \qq(k|\xx)$,
with $\qq(k \neq y|\xx)=0$ is truncated to a small value such that $\log(\qq(k \neq y|\xx)) = A$ (eg. $A=-4$). RCE has also been proved to be robust to label noise, and can be combined with CE to form the Symmetric Cross Entropy (SCE) for robust classification and boosted learning \cite{wang2019symmetric}.

Focal Loss (FL) \cite{lin2017focal}, originally proposed for dense object detection, is also an effective loss function for classification. FL is also a generalization of the CE loss, and is defined as: $FL = -\sum_{k=1}^{K} \qq(k|\xx) (1 - \pp(k|\xx))^{\gamma}\log \pp(k|\xx)$,
where $\gamma \geq 0$ is a tunable parameter. FL reduces to the CE loss when $\gamma=0$, and is not robust to noisy labels following \cite{ghosh2017robust}.

\subsection{Normalized Loss Functions}\label{normalized_loss}
Following \cite{ghosh2017robust,charoenphakdee2019symmetric}, we know that if a loss function $\L$ satisfies $\sum_{j}^{K} \L(f(\xx), j) = C, \forall \xx \in \X, \forall f$,
where $C$ is some constant, then $\L$ is noise tolerant under mild assumptions.
Based on this, we propose to normalize a loss function by:
\begin{equation}\label{eq:normalized_loss}
    \L_{\text{norm}} = \frac{\L(f(\xx), y)}{\sum_{j=1}^{K}\L(f(\xx), j)}.
\end{equation}
A normalized loss has the property: $\L_{\text{norm}} \in [0, 1]$.

Accordingly, we can normalize the above four loss functions defined in Section \ref{sec:preliminary} as follows.
The Normalized Cross Entropy (NCE) loss can be defined as:
\begin{equation}\label{eq:nce}
\begin{split}
    NCE &= \frac{-\sum_{k=1}^{K} \qq(k|\xx) \log \pp(k|\xx)}{-\sum_{j=1}^{K}\sum_{k=1}^{K} \qq(y=j|\xx) \log \pp(k|\xx)} \\
    &= \log_{\prod_{k}^{K} \pp(k|\xx)}\pp(y|\xx),
\end{split}
\end{equation}
where, the last equality holds following the change of base rule in logarithm (eg. $\log_a b = \frac{\log b}{\log a}$). 

The Normalized Mean Absolute Error (NMAE) is:
\begin{equation}\label{eq:nmae}
\begin{split}
    NMAE &= \frac{\sum_{k=1}^{K} |\pp(k|\xx) - \qq(k|\xx)|}{\sum_{j=1}^{K}\sum_{k=1}^{K} |\pp(k|\xx) - \qq(y=j|\xx)|} \\
    &= \frac{1}{K-1}(1 - \pp(y|\xx)) = \frac{1}{2(K-1)} \cdot MAE.
\end{split}
\end{equation}
The last two equalities hold due to $\sum_{k=1}^{K} |\pp(k|\xx) - \qq(k|\xx)| = 2(1 - \pp(y|\xx))$.
As can be observed, NMAE is simply a scaled version of MAE by a factor of $\frac{1}{2(K-1)}$.

The Normalized Reverse Cross Entropy (NRCE) loss is:
\begin{equation}\label{eq:nrce}
\begin{split}
    NRCE &= \frac{-\sum_{k=1}^{K} \pp(k|\xx) \log \qq(k|\xx)}{-\sum_{j=1}^{K}\sum_{k=1}^{K} \pp(k|\xx) \log \qq(y=j|\xx)} \\
    &= \frac{1}{K-1}(1 - \pp(y|\xx)) = \frac{1}{A(K-1)} \cdot RCE.
\end{split}
\end{equation}
The last two equalies hold as $\sum_{k=1}^{K} \pp(k|\xx) \log \qq(k|\xx) = A(1-\pp(y|\xx))$.
Similar to NMAE, NRCE is a scaled version of RCE by a factor of $\frac{1}{A(K-1)}$.

The Normalized Focal Loss (NFL) can be defined as:
\begin{equation}\label{eq:nfl}
\footnotesize
\begin{split}
    NFL &= \frac{-\sum_{k=1}^{K} \qq(k|\xx) (1 - \pp(k|\xx))^{\gamma} \log \pp(k|\xx)}{-\sum_{j=1}^{K}\sum_{k=1}^{K} \qq(y=j|\xx) (1 - \pp(k|\xx))^{\gamma} \log \pp(k|\xx)} \\
    &= \log_{\prod_{k}^{K} (1 - \pp(k|\xx))^{\gamma}\pp(k|\xx)}(1 - \pp(y|\xx))^{\gamma}\pp(y|\xx).
\end{split}
\end{equation}

Under this normalization scheme, the normalized forms of robust loss functions such as MAE and RCE are simply a scaled version of their original forms. This keeps their robustness property. For the rest of this paper, we will use the original forms for MAE and RCE if not otherwise explicitly stated. On the contrary, normalization on nonrobust loss functions such as CE and FL derives new loss functions.
Note that the above four normalized losses are just a proof-of-concept, other loss functions can also be normalized following Eq. \eqref{eq:normalized_loss}. 

\subsection{Theoretical Justification}
Following previous works \cite{ghosh2017robust,wang2019symmetric}, we can show that normalized loss functions are noise tolerant to both symmetric and asymmetric label noise.

\begin{lemma}\label{lemma_1}
In a multi-class classification problem, any normalized loss function $\L_{\text{norm}}$ is noise tolerant under symmetric (or uniform) label noise, if noise rate $\eta < \frac{K-1}{K}$.
\end{lemma}

\begin{lemma}\label{lemma_2}
In a multi-class classification problem, given $R(f^{*})=0$ and $0 \leq \L_{\text{norm}}(f(\xx), k) \leq \frac{1}{K-1}, \forall k$, any normalized loss function $\L_{\text{norm}}$ is noise tolerant under asymmetric (or class-conditional) label noise, if noise rate $\eta_{jk} < 1- \eta_y$.
\end{lemma}

Detailed proofs for Lemma \ref{lemma_1} and Lemma \ref{lemma_2} can be found in Appendix \ref{appendix_proof}. 
We denote the risk of classifier $f$ under clean labels as $R(f)=\E_{\xx, y^{*}}\L_{\text{norm}}$, and the risk under label noise rate $\eta$ as $R^{\eta}(f)=\E_{\xx, y} \L_{\text{norm}}$. Let $f^*$ and $f_{\eta}^*$ be the global minimizers of $R(f)$ and $R^{\eta}(f)$, respectively. We need to prove $f^*$ is also a global minimizer of noisy risk $R^{\eta}(f)$ for $\L$ to be robust.  
The noise rate conditions in Lemma \ref{lemma_1} ($\eta < \frac{K-1}{K}$) and Lemma \ref{lemma_2} ($\eta_{jk} < 1- \eta_y$) generally requires that the correct labels are still the majority of the class. 
In Lemma \ref{lemma_2}, the restrictive condition $R(f^{*})=0$ may not be satisfied in practice (eg. the classes may not completely separable), however, good empirical robustness can still be achieved.
While the condition $0 \leq \L_{\text{norm}}(f(\xx), k) \leq \frac{1}{K-1}, \forall k$ can be easily satisfied by a typical loss function. 
We refer the reader to \cite{charoenphakdee2019symmetric} for more discussions of other theoretical properties such as classification calibration.

So far, we have presented a somewhat surprising but theoretically justified result that any loss function can be made robust to noisy labels. This advances current theoretical progresses in this field. While this finding is exciting, in the following, we will empirically show that robustness alone is not sufficient for obtaining good performance.

\section{Robustness Alone is not Sufficient}\label{sec:pos_neg}
In this section, we empirically show that the above four robust loss functions (eg. NCE, NFL, MAE and RCE) all suffer from an underfitting problem, and thus are not sufficient by themselves to train accurate DNNs. We then propose a new framework to build loss functions that are both theoretically robust and learning sufficient.

\noindent\textbf{Robust losses can suffer from underfitting.}
To motivate this problem, we use an example on CIFAR-100 dataset with 0.6 symmetric noise. We train a ResNet-34 \cite{he2016deep} using both normalized and unnormalized loss functions (detailed setting can be found in Section \ref{sec:benckmark_robust}). As can be observed in Figure \ref{fig:1}, CE and FL losses become robust after normalization, however, this robustness does not lead to more accurate models. In fact, robust losses NCE and NFL demonstrate even worse performance than nonrobust CE and FL. Moreover, even without normalization, the originally robust loss functions MAE and RCE also suffer from underfitting: they even fail to converge in this scenario.
We find that this underfitting issue occur across different training settings in terms of learning rate, learning rate scheduler, weight decay and the number of training epochs.
We identify this problem as an \emph{underfitting problem of existing robust loss functions}, at least for the four tested loss functions (eg. NCE, NFL, MAE and RCE).
Next, we will propose a new loss framework to address this problem.

\begin{figure}[!ht]
	\centering
	\begin{subfigure}{0.48\linewidth}
		\includegraphics[width=\textwidth]{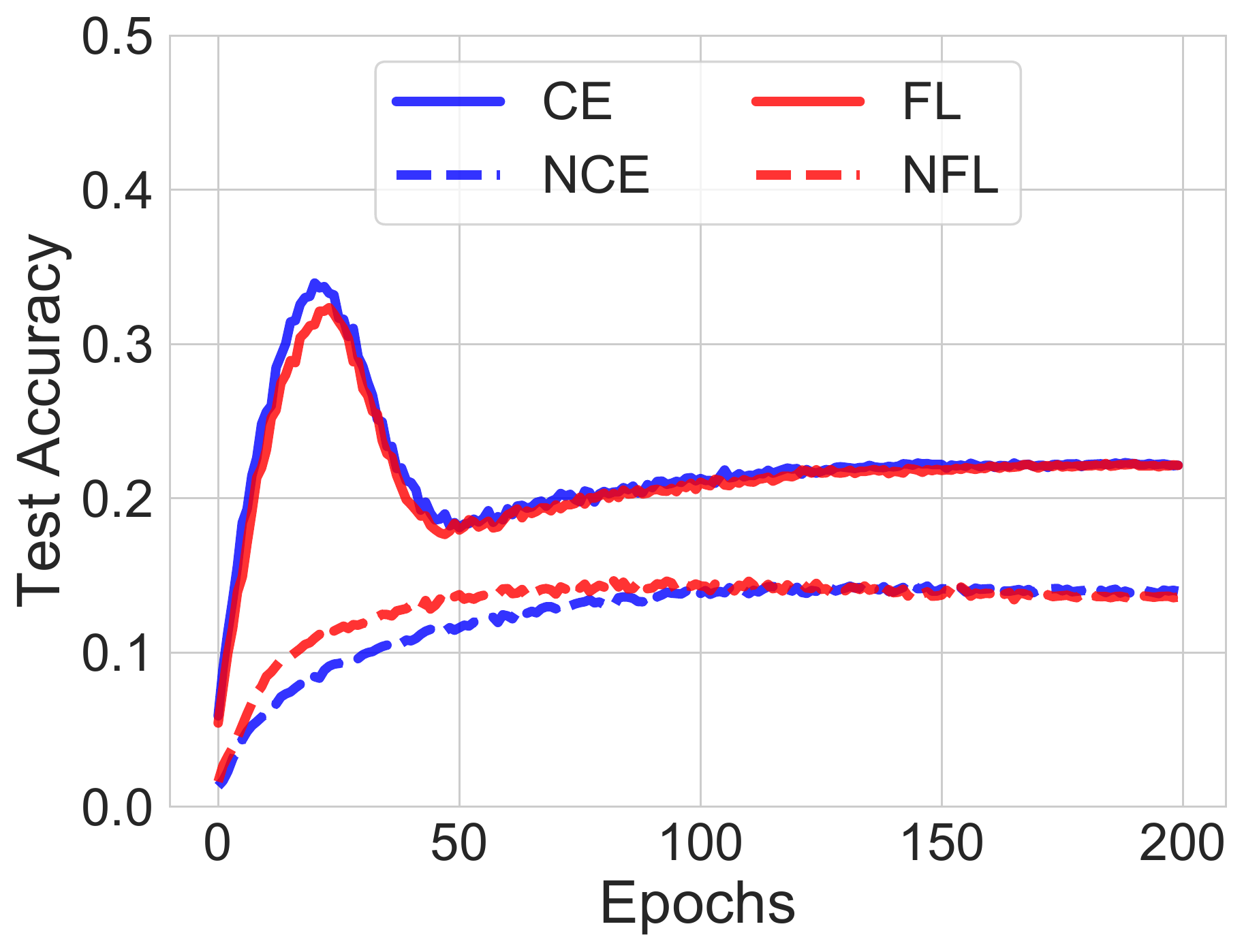}
		\label{ce_nce_100}
	\end{subfigure}
	\begin{subfigure}{0.48\linewidth} 
		\includegraphics[width=\textwidth]{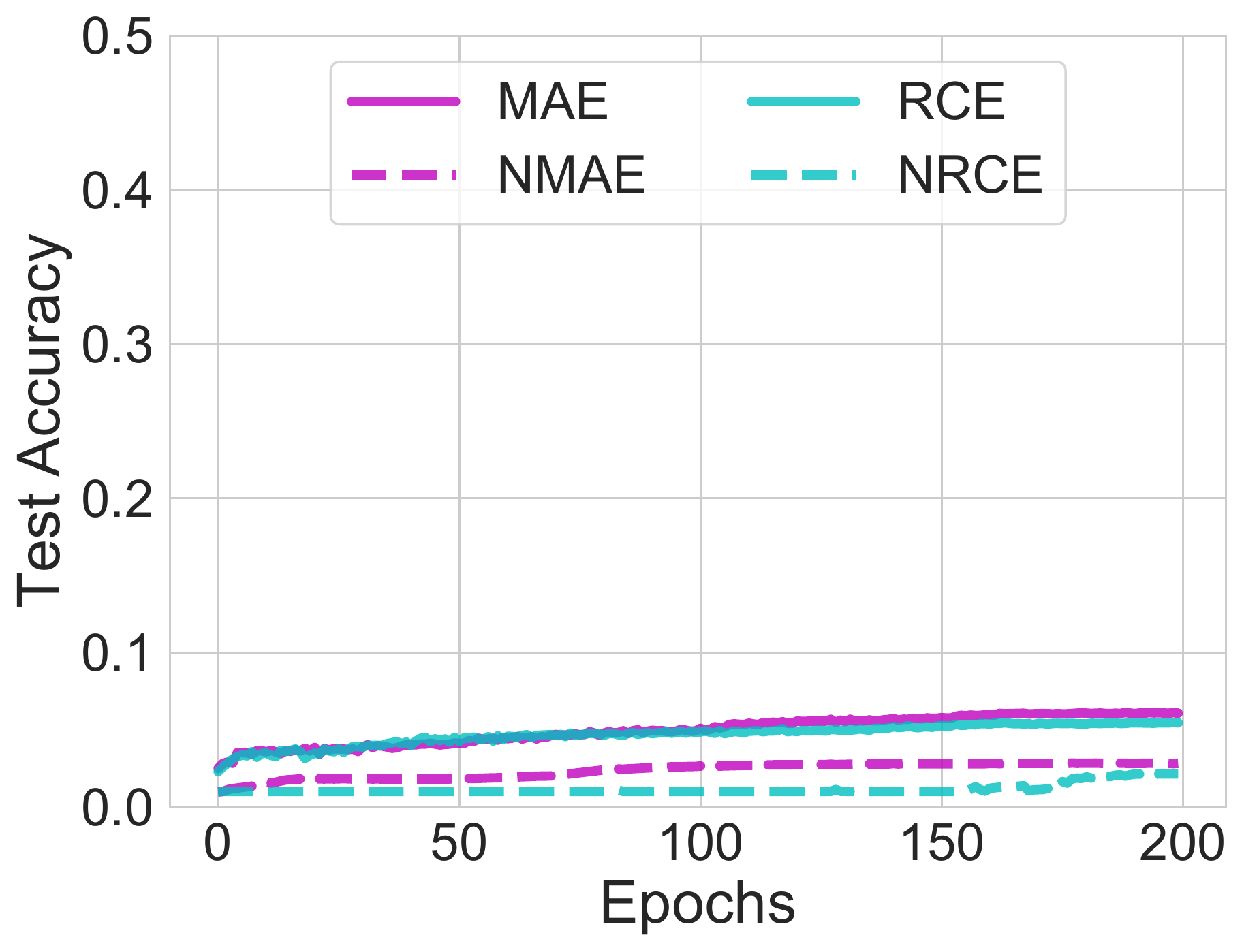}
		\label{fl_nfl_100}
	\end{subfigure}
	\vspace{-0.15 in}
	\caption{Test accuracies of unnormalized versus normalized loss functions on CIFAR-100 under 0.6 symmetric noise.}
	\vspace{-0.15 in}
	\label{fig:1}
\end{figure}

\subsection{Proposed Active Passive Loss (APL)}\label{sec:pos_neg}
In \cite{kim2019nlnl}, the use of complementary labels (``input does not belong to this complementary class") together with the original labels was shown to help learning and robustness.
In \cite{wang2019symmetric}, a Reverse Cross Entropy term was found can provide a robust boost to the CE loss. To generalize these works taking a loss function perspective, we characterize existing robust functions into two types: ``Active" and ``Passive", based on their optimization (maximization/minimization) behaviors.

At a high level, a loss is defined ``Active" if it only optimizes at $\qq(k = y|x)=1$, otherwise, a loss is defined as ``Passive". We denote the basic function of loss $\L(f(x), y)$ by $\ell(f(x), k)$, that is $\L(f(x), y) = \sum_{k=1}^{K} \ell(f(x), k)$. Then, we can define the active and passive loss functions as:

\begin{definition}{(Active loss function)}
$\L_{\text{Active}}$ is an active loss function if $\forall (\xx, y) \in \D \; \forall k \neq y \;\ell(f(\xx), k) = 0$.
\end{definition}

\begin{definition}{(Passive loss function)}
$\L_{\text{Passive}}$ is a passive loss function if
$\forall (\xx, y) \in \D \; \exists k \neq y \; \ell(f(\xx), k) \neq 0 $.
\end{definition}

According to the above two definitions, active losses only explicitly maximize the network's output probability at the class position specified by the label $y$. For example in CE loss, only the probability at $\qq(k=y|x)=1$ is explicitly maximized (the loss is zero at $\qq(k \neq y|x)=0$). Different from active losses, passive losses also explicitly minimize the probability at at least one other class positions. For example in MAE, the probabilities at position $k \neq y$ are also explicitly minimized along with 
the maximization of the probability at $k=y$. Note that this characterization applies to both robust and nonrobust loss functions. Table \ref{tab:summary} summarizes examples of active and passive losses.

\begin{table}[!ht]
\vspace{-0.1 in}
\caption{Examples of active and passive loss functions.}
\label{tab:summary}
\centering
\begin{adjustbox}{width=1\linewidth}
\small
\begin{tabular}{l|c|c}
\hline
Loss Type & Active & Passive \\ \hline
 Examples & CE, NCE, FL, NFL & MAE, NMAE, RCE, NRCE \\
 \hline
\end{tabular}
\end{adjustbox}
\vspace{-0.1 in}
\end{table}

\noindent\textbf{Definition of APL.} Inspired by the benefit of symmetric \cite{wang2019symmetric} or complementary learning \cite{kim2019nlnl}, we propose to combine a robust active loss and a robust passive loss into an ``Active Passive Loss" (APL) framework for both robust and sufficient learning. Formally,
\begin{equation}\label{eq:apl}
    \L_{\text{APL}} = \alpha \cdot \L_{\text{Active}} + \beta \cdot \L_{\text{Passive}},
\end{equation}
where, $\alpha, \beta > 0$ are parameters to balance the two terms.
An important requirement for the two loss terms is robustness, which means a nonrobust loss should be normalized following Eq. \eqref{eq:normalized_loss} for it to be used within our APL scheme. This guarantees the robustness property of APL loss functions (proof can be found in Appendix \ref{appendix_proof}):

\begin{lemma}\label{lemma_3}
$\forall \alpha, \forall \beta$, if $\L_{\text{Active}}$ and $\L_{\text{Passive}}$ are noise tolerant, then $\L_{\text{APL}} = \alpha \cdot \L_{\text{Active}} + \beta \cdot \L_{\text{Passive}}$ is noise tolerant.
\end{lemma}

In APL, the two loss terms optimize the same objective from two complementary directions (eg. maximizing $\pp(k=y|\xx)$ and minimizing $\pp(k \neq y|\xx)$). 
For the four loss functions considered in this paper, there are four possible combinations satisfying our APL principle: 1) $\alpha NCE+ \beta MAE$, 2) $\alpha NCE+\beta RCE$, 3) $\alpha NFL+ \beta MAE$ and 4) $\alpha NFL+ \beta RCE$. For simplicity we omit the parameters $\alpha, \beta$ in the rest of this paper. 
According to our active/passive definitions, APL losses can be considered as passive losses. However, APL losses are different from passive losses that have only one term, since they contain at least two terms and one of them is an active loss term.
Whilst different choices of the two loss terms may lead to different performance, we will show in Section \ref{sec:experiments} that APL losses generally achieve better or at least comparable performance to state-of-the-art noisy label learning methods.

\subsection{More Insights into APL Loss Functions}
Here, we provide some insights into the underfitting issue of robust loss functions, and why the proposed APL losses can address underfitting.

\noindent\textbf{Why robust loss functions underfit?}
Taking the NCE loss defined in Eq. \eqref{eq:nce} as an example, the underfitting is caused by the extra terms introduced into the denominator by the normalization. In Eq. \eqref{eq:nce}, NCE is in the form of $\frac{P}{P+Q}$, where $P=-\log(p_y)$ and $Q=-\sum_{k \neq y}\log(p_k)$. During training, the $Q$ term may increase even when $P$ is fixed (eg. $p_y$ is fixed), and it reaches the highest value when all $p_{k \neq y}$ equals to $(1-p_y)/(K-1)$ (eg. the highest entropy). This implies that the network may learn nothing for the prediction (as $p_y$ is fixed) even when the loss decreases (as $Q$ increases). This tends to hinder the convergence and cause the underfitting problem. Other robust loss functions such as MAE and RCE all suffer from a similar issue.

\noindent\textbf{Why APL can address underfitting?}
APL combines an active loss with a passive loss. By definition, the passive loss explicitly minimizes (at least one component of) the Q term discussed above so that it won’t increase when $p_y$ is fixed. This directly addresses the underfitting issue of a robust active loss. Therefore, APL losses can leverage both the robustness and the convergence advantages.
Note that, by definition, passive loss has a broader scope than active loss. A single passive loss like MAE can be decomposed into an active term and a passive term, with the two terms already form an APL loss. With proper balancing between the two terms, the reformulated MAE can also be a powerful new loss. For example, a recent work has shown that a reweighted MAE can outperform CE \cite{wang2019imae}.

\subsection{Connection to Related Work}
Our APL framework is a generalization of several state-of-the-art methods. Following APL, better performance can be achieved with existing loss functions, rather than complex modifications on the training procedure.
Although NLNL \cite{kim2019nlnl} can improve robustness with complementary labels, it has slow convergence (10$\times$ slower than standard training), and requires a complex 3-stage training procedure: 1) training with complementary labels, 2) training with high confidence (above a threshold) complementary labels, and 3) training with high confidence original labels. From our APL perspective, NLNL switches back and forth between active learning (with original labels) and passive learning (with complementary labels). Such a learning scheme can instead be achieved alternatively using our APL. 
Indeed, when defined on complementary labels, the CE loss becomes -1/(C-1)log(1-RCE) with A=-1 in RCE, and our APL loss NCE+RCE can be seen as a simpler alternative for NLNL.
Compared to the SCE \cite{wang2019symmetric} loss (eg. CE+RCE), our APL loss NCE+RCE can be seen as its normalized version, which has theoretically guaranteed robustness. This modification to SCE can improve its performance considerably (see Section \ref{sec:benckmark_robust}).
Compared to the GCE loss \cite{zhang2018generalized} which can be regraded as a mixture of CE and MAE, our APL loss NCE+MAE is an alternative solution that directly adds the two terms together with normalization. NCE+MAE is theoretically robust while GCE is not. Moreover, the GCE loss itself can be normalized and improved following our APL framework (see Section \ref{sec:ngce}).

\section{Experiments}\label{sec:experiments}
In this section, we empirically investigate our proposed APL loss functions on benchmark datasets MNIST \cite{lecun1998gradient}, CIFAR-10/-100
\cite{krizhevsky2009learning}, and a real-world noisy dataset WebVision \cite{li2017webvision}.

\subsection{Empirical Understandings}\label{sec:understanding}

\noindent\textbf{Normalized losses are robust.}
We first run a set of experiments on CIFAR-10 and CIFAR-100 to verify whether non-robust losses CE and FL become robust after normalization (NCE and NFL). We set the label noise to be symmetric, and the noise rate to 0.6 for both CIFAR-10 and CIFAR-100. We use an 8-layer convolutional neural network (CNN) for CIFAR-10 and a ResNet-34 \cite{he2016deep} for CIFAR-100. On each dataset, we train the same network using different loss functions, eg. normalized versus unnormalized. For FL/NFL loss we set $\gamma=0.5$, while for RCE/NRCE loss, we set $A=-4$. 
Detailed settings are in Section \ref{sec:benckmark_robust}.

As shown in Figures \ref{fig:1} \& \ref{fig:2}, both CE and FL losses exhibit significant overfitting after epoch 25. However, as we have theoretically proved, their normalized forms (eg. NCE and NFL) are robust: no overfitting was observed during the entire training process. Moreover, for the already robust loss functions MAE and RCE, normalization does not break their robustness property. We observe the same results across different datasets (eg. MNIST, CIFAR-10 and CIFAR-100) under different noise rates ($\eta \in [0.2, 0.8]$). In general, the higher the noise rate, the more overfitting of nonrobust loss functions, and their normalized forms are always robust. This empirically verifies our theoretical finding that any loss can be made robust following normalization in Eq. \eqref{eq:normalized_loss}.

\noindent\textbf{Can scaling help sufficient learning?}
As one may have noticed in Figures \ref{fig:1} \& \ref{fig:2}, NMAE and NRCE exhibit more severe underfitting than MAE and RCE, even though they are just scaled versions of MAE and RCE. This raises the question: can the underfitting problem be addressed by scaling the normalized losses up by a factor?
In Figure \ref{fig:3}, we show different scales applied to NCE, NFL, MAE and RCE for training on CIFAR-100 with 0.6 symmetric noise. We find that scaled NCE and NFL only slightly improve learning after epoch 150, when the learning rate is decayed to be smaller. This is because scaling the loss is equivalent to scaling the gradients, a similar effect to increasing the learning rate. Moreover, scaled MAE and RCE still fail to converge in this scenario.
This highlights that scaling may not be an effective solution for sufficient learning, especially for challenging datasets like CIFAR-100. On the simple dataset CIFAR-10, proper scaling does help learning. But this can alternatively can be achieved by adjusting the learning rate.

\begin{figure}[!t]
	\centering
	\begin{subfigure}{0.48\linewidth}
		\includegraphics[width=\textwidth]{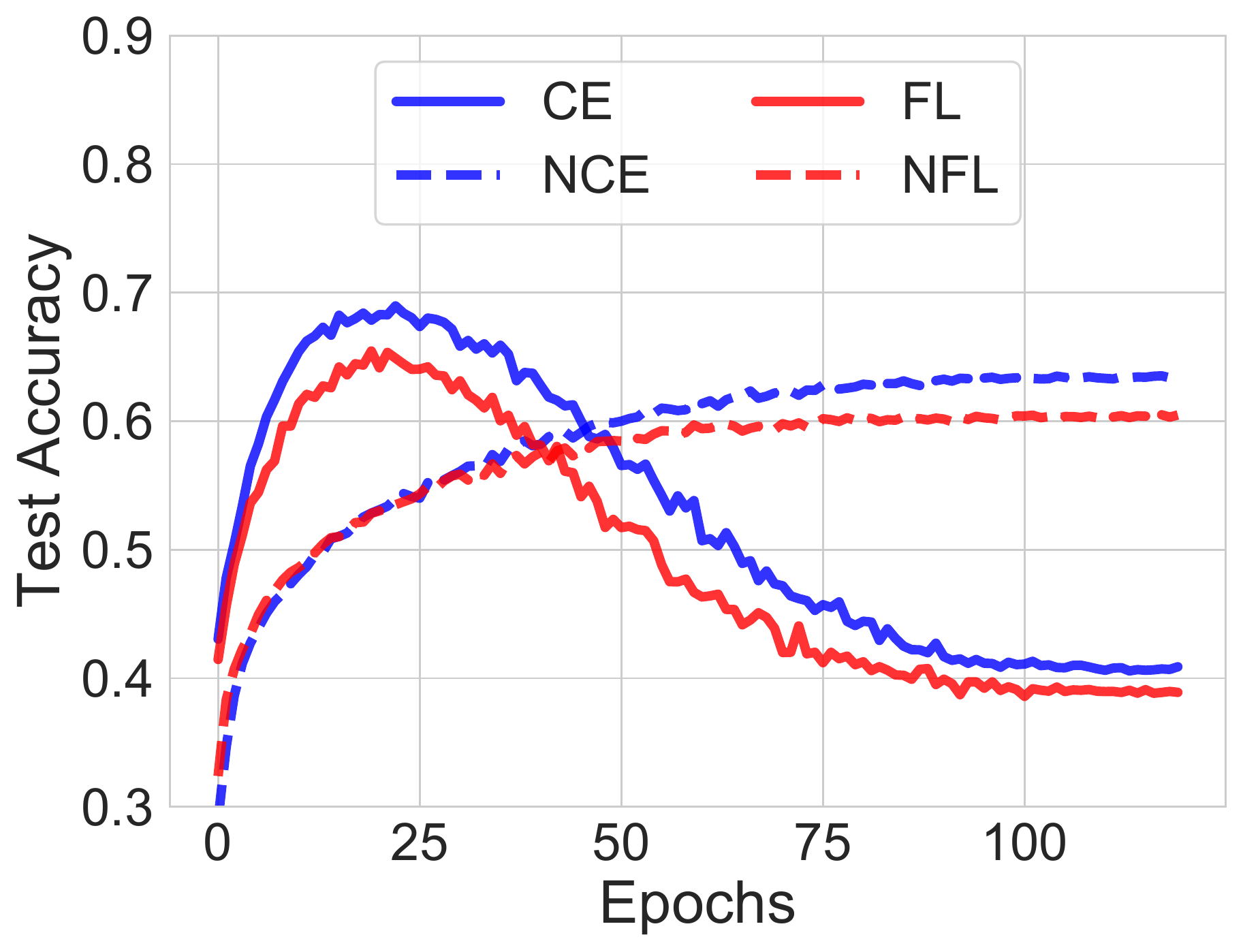}
		\label{ce_nce}
	\end{subfigure}
	\begin{subfigure}{0.48\linewidth} 
		\includegraphics[width=\textwidth]{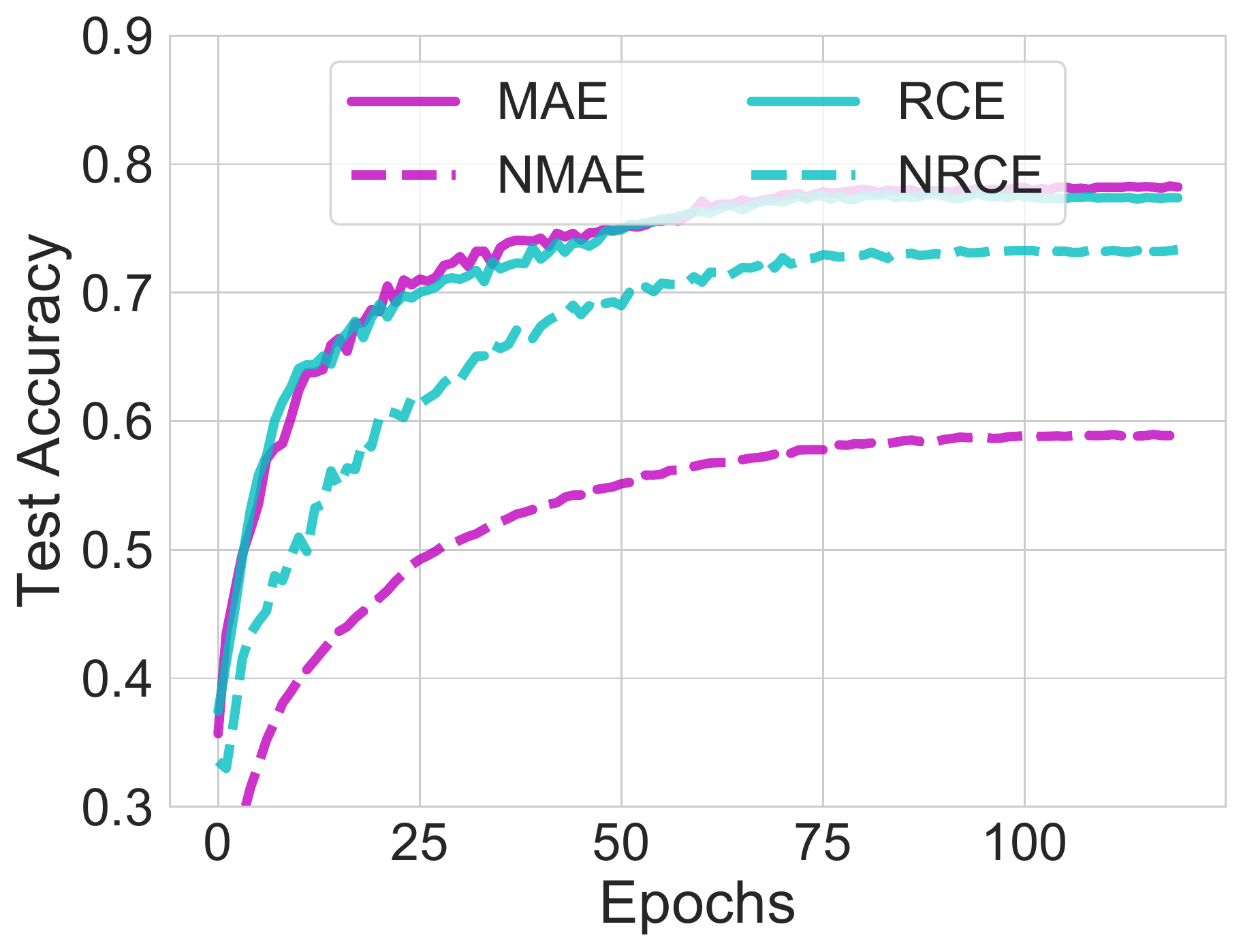}
		\label{fl_nfl}
	\end{subfigure}
	\vspace{-0.25 in}
	\caption{Test accuracies of unnormalized versus normalized loss functions on CIFAR-10 under 0.6 symmetric noise. }
	\label{fig:2}
\end{figure}

\begin{figure}[!t]
	\centering
	\begin{subfigure}{0.48\linewidth}
		\includegraphics[width=\textwidth]{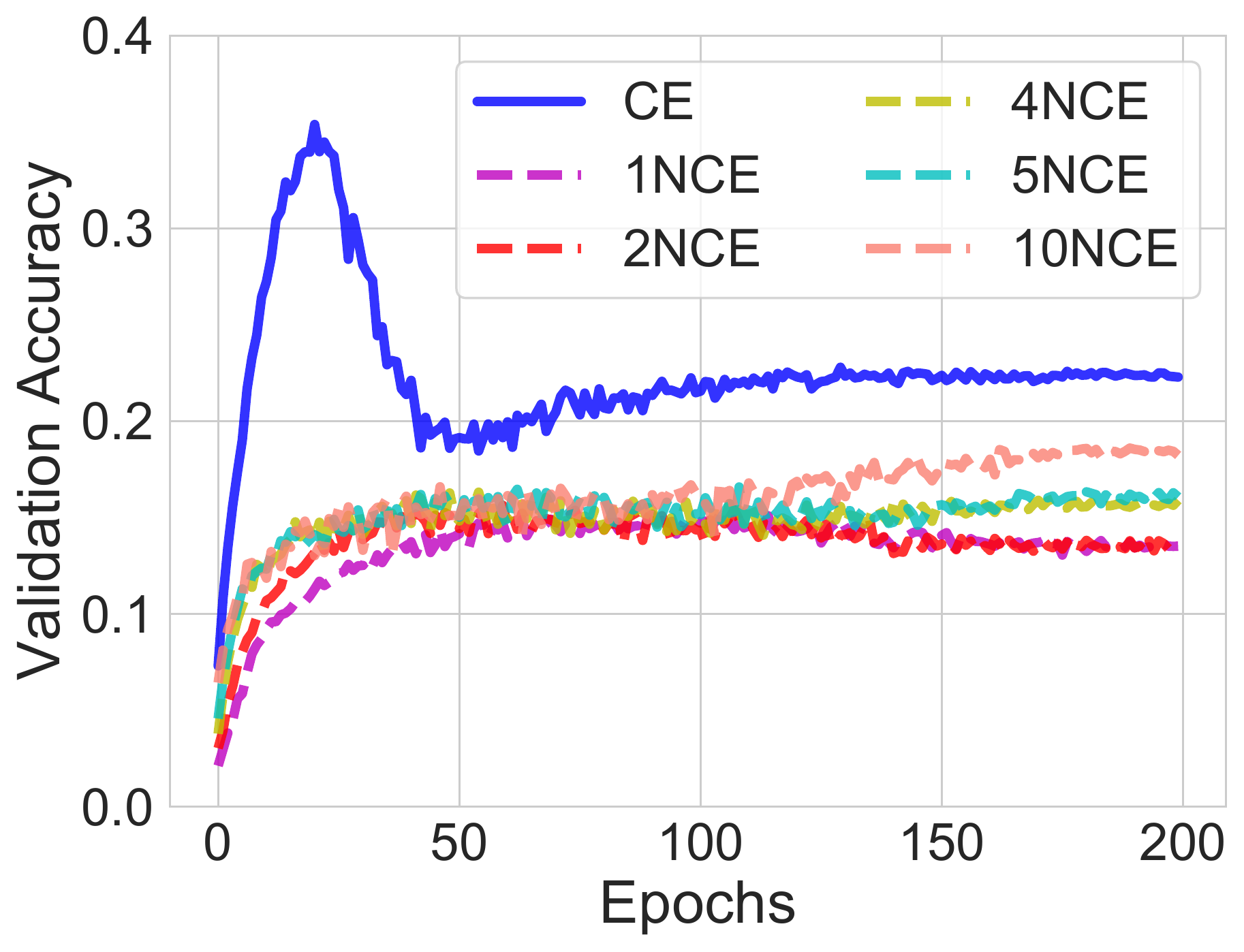}
		\caption{Scaled NCE}
		\label{nce_scale}
	\end{subfigure}
	\begin{subfigure}{0.48\linewidth} 
		\includegraphics[width=\textwidth]{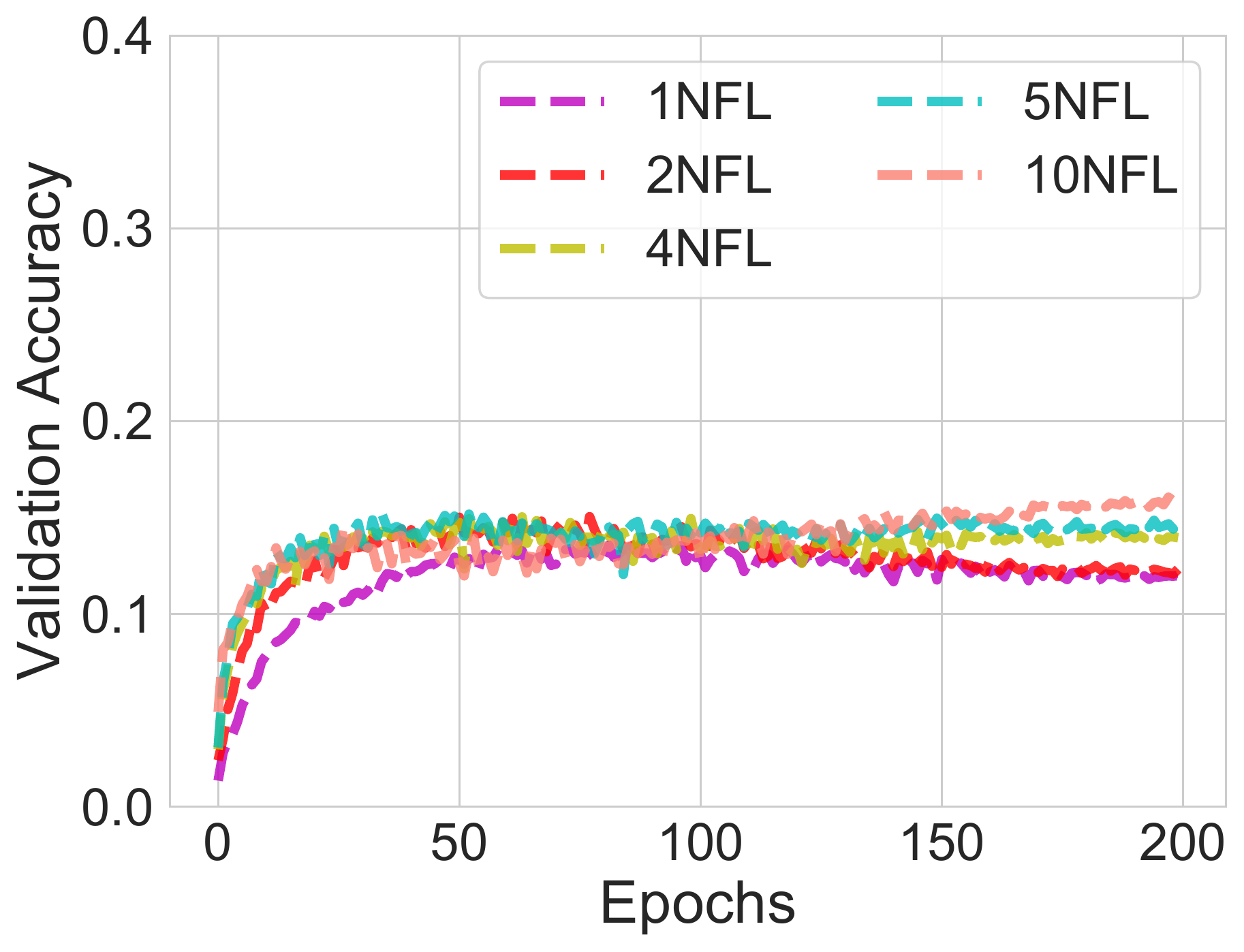}
		\caption{Scaled NFL}
		\label{nfl_scale}
	\end{subfigure}\\
	\begin{subfigure}{0.48\linewidth}
		\includegraphics[width=\textwidth]{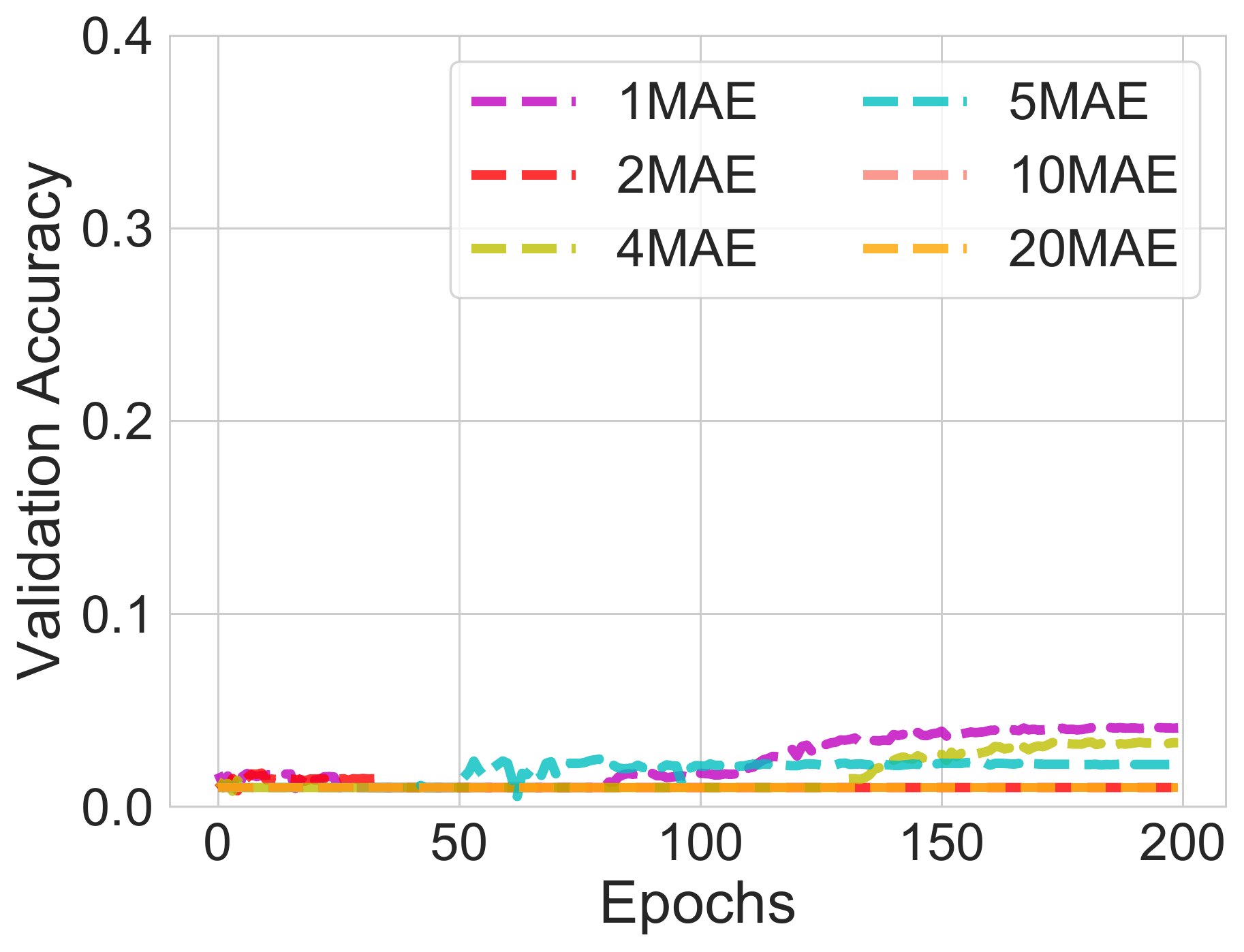}
		\caption{Scaled MAE}
		\label{mae_scale}
	\end{subfigure}
	\begin{subfigure}{0.48\linewidth} 
		\includegraphics[width=\textwidth]{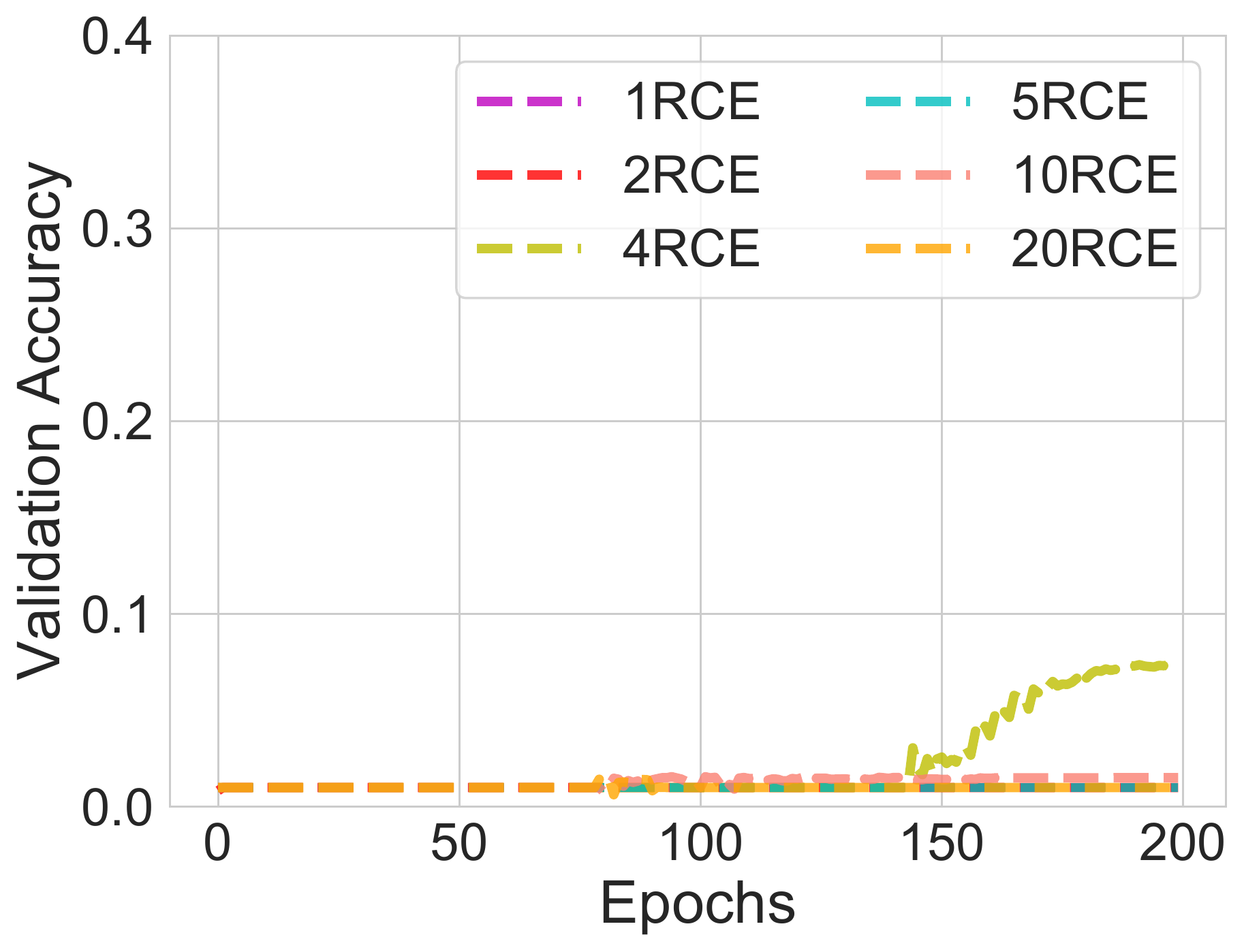}
		\caption{Scaled RCE}
		\label{rce_scale}
	\end{subfigure}
	\vspace{-0.1 in}
	\caption{Test accuracies of scaled loss functions on CIFAR-100 with 0.6 symmetric noise.}
 	\vspace{-0.15 in}
	\label{fig:3}
\end{figure}

\noindent\textbf{APL losses are both robust and learning sufficient.}
We show the effectiveness of ``Active+Passive" learning, compared to other forms of combinations. We run experiments on CIFAR-10 and CIFAR-100 under the same settings as above. The parameters $\alpha, \beta$ for our APL are simply set to 1.0 without any tuning.
As shown in Figure \ref{fig:4}, the 4 APL loss functions demonstrate a clear advantage over either AAL (``Active+Active Loss'') or PPL (``Passive+Passive Loss''), especially for sufficient learning (high accuracy).
The AAL and PPL loss functions are robust but still suffer from the underfitting problem. This highlights that the overfitting and underfitting problems can be addressed simultaneously by the joint of active and passive losses by our APL.

\textbf{Parameter Analysis of APL.} We tune the parameters $\alpha$ and $\beta$ for NCE+RCE loss, then directly use these parameters for all other APL losses. This is also done on CIFAR-10 and CIFAR-100 datasets under 0.6 symmetric noise. 
We test the combinations between $\alpha \in \{0.1, 1.0, 10.0\}$ and $\beta \in \{0.1, 1.0, 10.0, 100.0\}$, then select the optimal combination according to the validation accuracy on a randomly sampled validation set (20\% training data). As shown in Figure \ref{fig:5}, the optimal parameters for CIFAR-10 are $\alpha=1, \beta=1$, and CIFAR-100 are $\alpha=10, \beta=0.1$. In general, on more complex dataset (eg. CIFAR-100 $>$ CIFAR-10), it requires more active learning (eg. a large $\alpha$) and less passive learning (eg. a small $\beta$) to achieve good performance.

\begin{figure}[!t]
	\centering
	\begin{subfigure}{0.48\linewidth}
		\includegraphics[width=\textwidth]{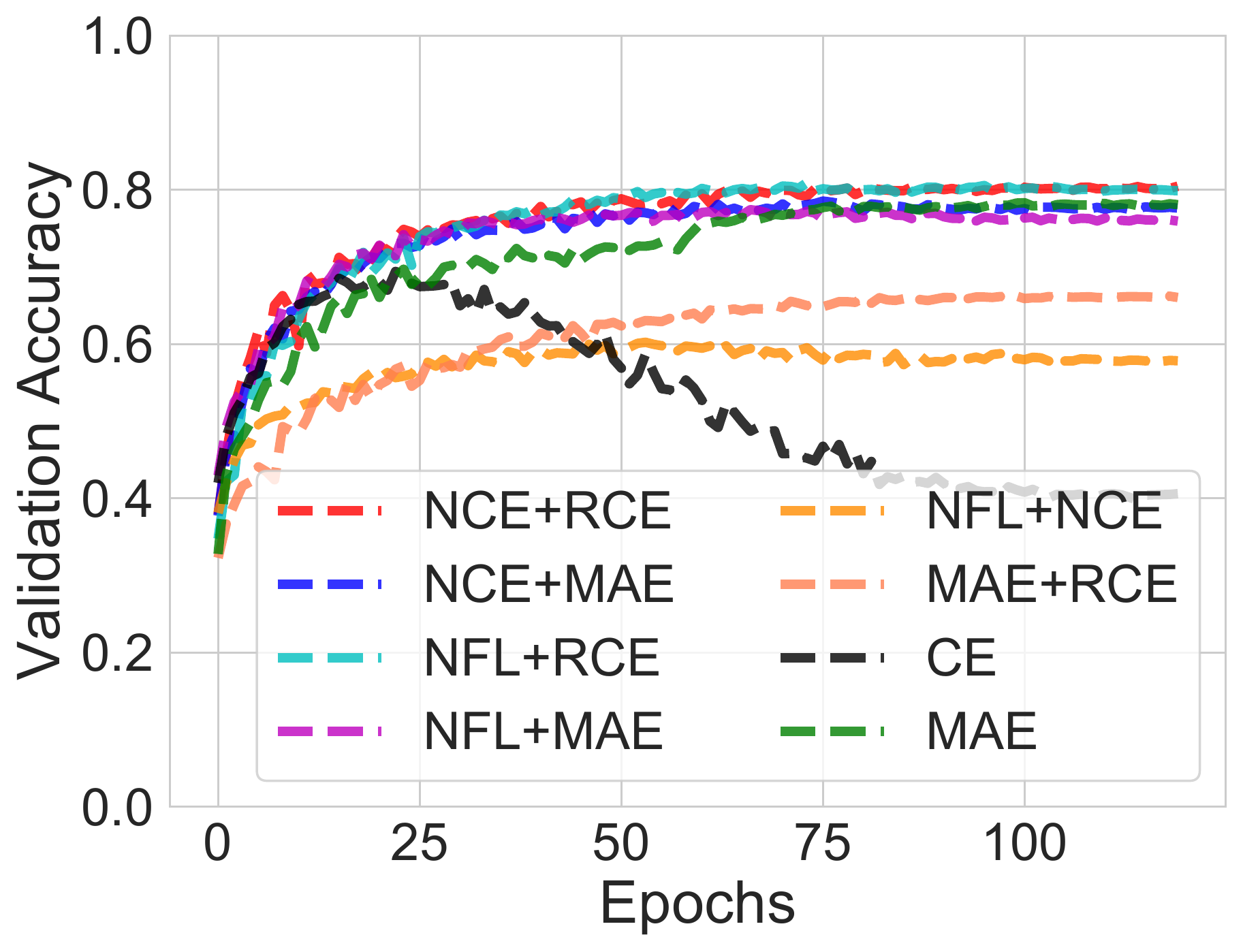}
		\caption{CIFAR-10}
		\label{ce_nce_combo}
	\end{subfigure}
	\begin{subfigure}{0.48\linewidth} 
		\includegraphics[width=\textwidth]{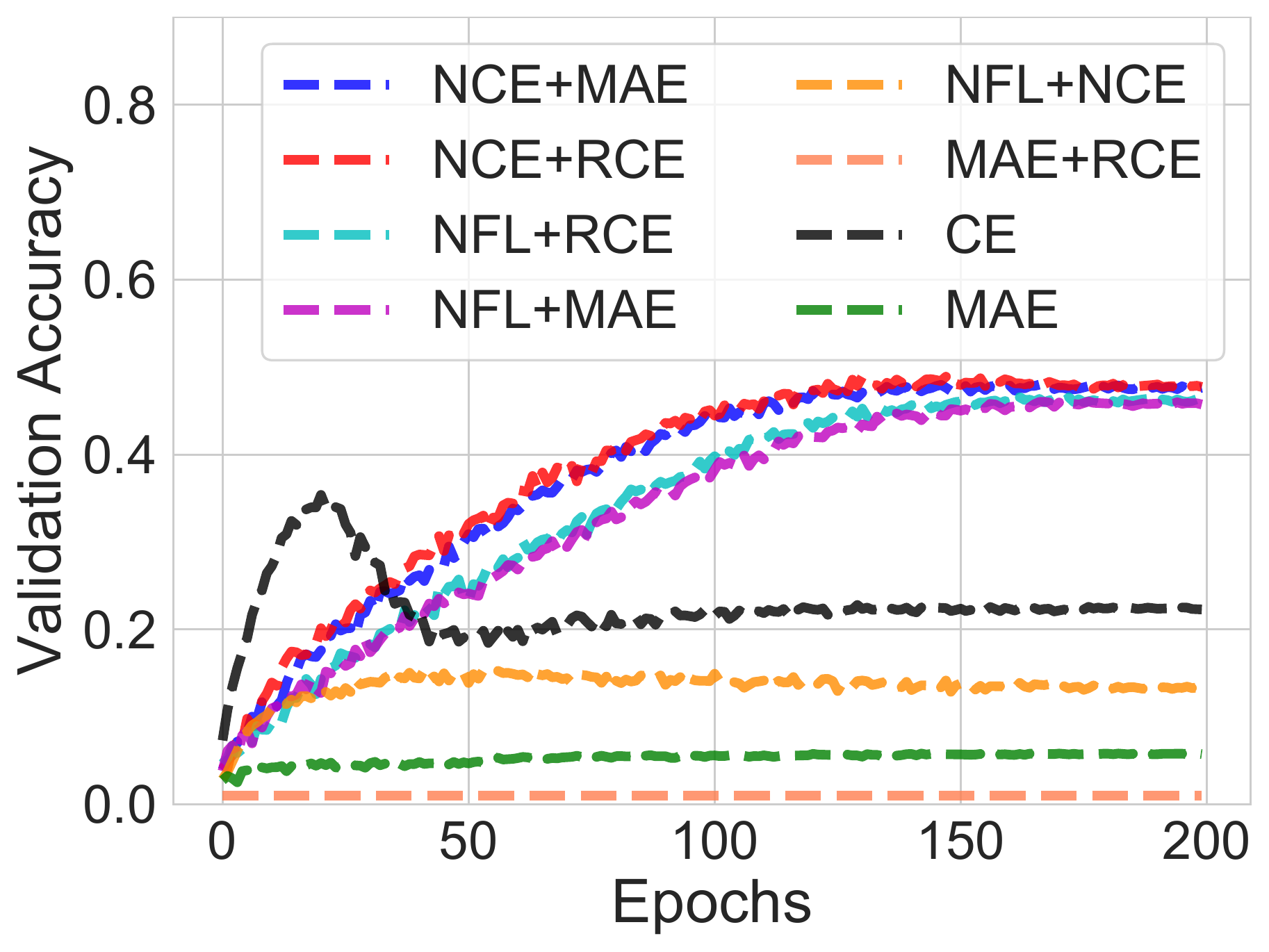}
		\caption{CIFAR-100}
		\label{fl_nfl_combo}
	\end{subfigure}
	\vspace{-0.1 in}
	\caption{Test accuracies of APL loss functions (NCE+MAE, NCE+RCE, NFL+MAE and NFL+RCE) versus ``AAL" loss (NFL+NCE) or ``PPL" loss (MAE+RCE) on CIFAR-10/CIFAR-100 with 0.6 symmetric noise.}
	\label{fig:4}
\end{figure}

\begin{figure}[!t]
	\centering
	\begin{subfigure}{0.48\linewidth}
		\includegraphics[width=\textwidth]{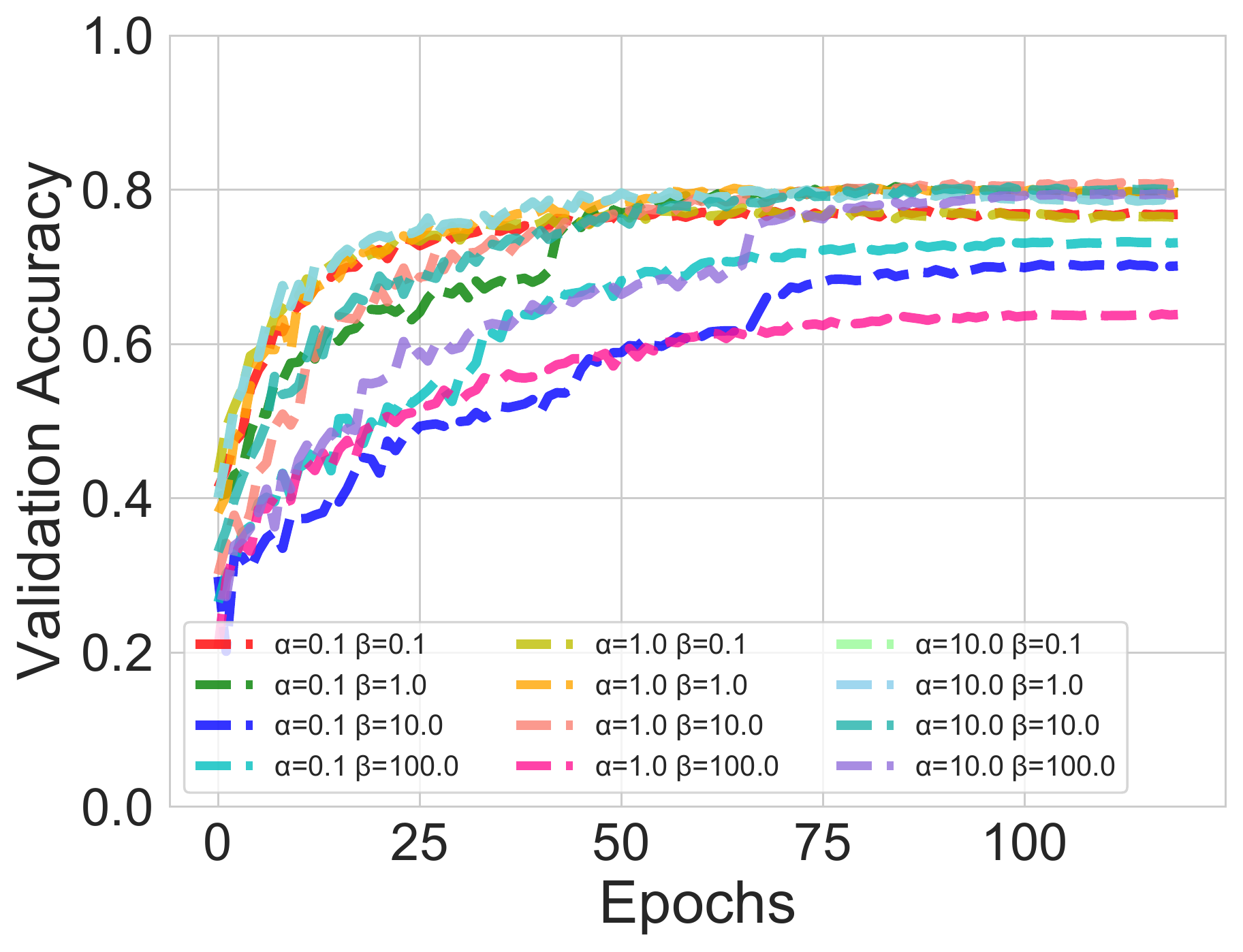}
		\caption{CIFAR-10 ($\eta=0.6$)}
	\end{subfigure}
	\begin{subfigure}{0.48\linewidth}
		\includegraphics[width=\textwidth]{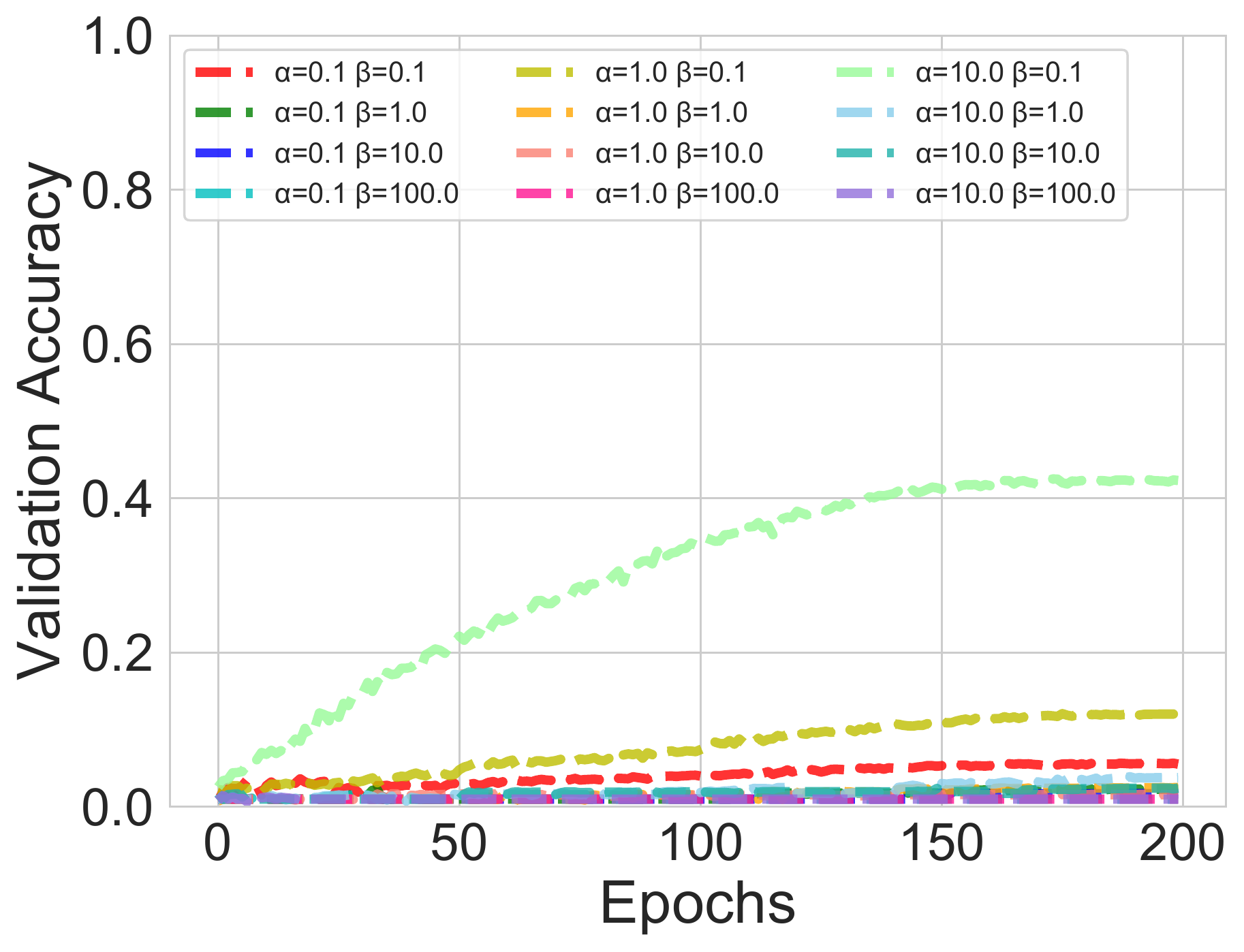}
		\caption{CIFAR-100 ($\eta=0.6$)}
	\end{subfigure}
	\vspace{-0.1 in}
	\caption{Validation accuracy of NCE+RCE loss with different parameters on CIFAR-10 and CIFAR-100 under symmetric noise.}
 	\vspace{-0.15 in}
	\label{fig:5}
\end{figure}

\begin{table*}[!ht]
\caption{Test accuracies (\%) of different methods on benchmark datasets with clean or symmetric label noise ($\eta \in [0.2, 0.8]$). The results (mean$\pm$std) are reported over 3 random runs and the top 2 best results are \textbf{boldfaced}.}
\label{tab:sym_robustness}
\centering
\small
\begin{tabular}{c|c|c|cccc}
\hline
\multirow{2}{*}{Datasets} & \multirow{2}{*}{Methods} & \multirow{2}{*}{Clean ($\eta$=0.0)} & \multicolumn{4}{c}{Symmetric Noise Rate ($\eta$)}\\
& & & 0.2 & 0.4 & 0.6 & 0.8 \\ \hline \hline
\multirow{10}{*}{\begin{tabular}[c]{@{}c@{}}\\MNIST\\ \end{tabular}} 
 & CE & $ 99.25\pm0.08 $  & $ 97.42\pm0.06 $  & $ 94.21\pm0.54 $  & $ 86.00\pm1.48 $  & $ 47.08\pm1.15 $  \\
 & FL & $ 99.30\pm0.02 $  & $ 97.45\pm0.19 $  & $ 94.71\pm0.25 $  & $ 85.76\pm1.85 $  & $ 49.77\pm2.26 $ \\
 & GCE & $ 99.27\pm0.01 $  & $99.18\pm0.06 $  & $ 98.72\pm0.05 $  & $ 97.43\pm0.23 $  & $ 12.77\pm2.00 $ \\
 & NLNL & $ 99.27\pm0.02 $  & $ 97.49\pm0.30 $  & $ 96.64\pm0.52 $  & $ 97.22\pm0.06 $  & $ 10.32\pm0.73 $ \\
 & SCE & $ 99.24\pm0.08 $  & $ 99.15\pm0.04 $  & $ 98.78\pm0.09 $  & $ 97.45\pm0.29 $  & $ 73.70\pm0.84 $ \\
 \cline{2-7}
& \textbf{NFL+MAE} & $ 99.39\pm0.04 $  & $ 99.12\pm0.06 $  & $ 98.74\pm0.14 $  & $ 96.91\pm0.09 $  & $ \boldsymbol{74.98\pm1.99} $ \\
& \textbf{NFL+RCE} & $ 99.38\pm0.02 $  & $ \boldsymbol{99.19\pm0.06} $  & $ \boldsymbol{98.79\pm0.10} $  & $ \boldsymbol{97.46\pm0.03} $  & $ 74.59\pm2.23 $ \\
& \textbf{NCE+MAE} & $ 99.37\pm0.02 $  & $ 99.14\pm0.05 $  & $ 98.78\pm0.00 $  & $ 96.76\pm0.34 $  & $ 74.66\pm1.11 $ \\
& \textbf{NCE+RCE} & $ 99.37\pm0.02 $  & $ \boldsymbol{99.20\pm0.04} $  & $ \boldsymbol{98.79\pm0.12} $  & $ \boldsymbol{97.48\pm0.13} $  & $ \boldsymbol{75.18\pm1.19} $ \\
\hline
\multirow{10}{*}{\begin{tabular}[c]{@{}c@{}}\\CIFAR-10\\ \end{tabular}} 
 & CE & $ 90.36\pm0.03 $  & $ 75.90\pm0.28 $  & $ 60.28\pm0.27 $  & $ 40.90\pm0.35 $  & $ 19.65\pm0.46 $ \\
 & FL & $ 89.63\pm0.25 $  & $ 74.59\pm0.49 $  & $ 57.55\pm0.39 $  & $ 38.91\pm0.62 $  & $ 19.43\pm0.27 $ \\
 & GCE & $ 89.38\pm0.23 $  & $ 87.27\pm0.21 $  & $ 83.33\pm0.39 $  & $ 72.00\pm0.37 $  & $ 29.08\pm0.80 $ \\
  & NLNL & $ 91.93\pm0.20 $  & $ 83.98\pm0.18 $  & $ 76.58\pm0.44 $  & $ 72.85\pm0.39 $  & $ 51.41\pm0.85 $ \\
 & SCE & $ 91.30\pm0.22 $  & $ 88.05\pm0.26 $  & $ 82.06\pm0.24 $  & $ 66.08\pm0.25 $  & $ 30.69\pm0.63 $ \\
 \cline{2-7}
& \textbf{NFL+MAE} & $ 89.25\pm0.19 $  & $ 87.33\pm0.14 $  & $ 83.81\pm0.06 $  & $ 76.36\pm0.31 $  & $ 45.23\pm0.52 $ \\
& \textbf{NFL+RCE} & $ 90.91\pm0.02 $  & $ \boldsymbol{89.14\pm0.13} $  & $ \boldsymbol{86.05\pm0.12} $  & $ \boldsymbol{79.78\pm0.13} $  & $ \boldsymbol{55.06\pm1.08} $ \\
& \textbf{NCE+MAE} & $ 88.83\pm0.34 $  & $ 87.12\pm0.21 $  & $ 84.19\pm0.43 $  & $ 77.61\pm0.05 $  & $ 49.62\pm0.72 $ \\
& \textbf{NCE+RCE} & $ 90.76\pm0.22 $  & $ \boldsymbol{89.22\pm0.27} $  & $ \boldsymbol{86.02\pm0.09} $  & $ \boldsymbol{79.78\pm0.50} $  & $ \boldsymbol{52.71\pm1.90} $ \\
\hline \hline
\multirow{10}{*}{\begin{tabular}[c]{@{}c@{}}\\CIFAR-100\\ \end{tabular}} 
 & CE & $ 70.89\pm0.22 $  & $ 56.99\pm0.41 $  & $ 41.40\pm0.36 $  & $ 22.15\pm0.40 $  & $ 7.58\pm0.44 $ \\
 & FL & $ 70.61\pm0.44 $  & $ 56.10\pm0.48 $  & $ 40.77\pm0.62 $  & $ 22.14\pm1.00 $  & $ 7.21\pm0.25 $ \\
 & GCE & $ 69.00\pm0.56 $  & $ 65.24\pm0.56 $  & $ 58.94\pm0.50 $  & $ 45.18\pm0.93 $  & $ 16.18\pm0.46 $ \\
 & NLNL & $ 68.72\pm0.60 $  & $ 46.99\pm0.91 $  & $ 30.29\pm1.64 $  & $ 16.60\pm0.90 $  & $ 11.01\pm2.48 $ \\
 & SCE & $ 70.38\pm0.45 $  & $ 55.39\pm0.18 $  & $ 39.99\pm0.59 $  & $ 22.35\pm0.65 $  & $ 7.57\pm0.28 $ \\
 \cline{2-7}
& \textbf{NFL+MAE} & $ 67.98\pm0.52 $  & $ 63.58\pm0.09 $  & $ 58.18\pm0.08 $  & $ 46.10\pm0.50 $  & $ 24.78\pm0.82 $ \\
& \textbf{NFL+RCE} & $ 68.23\pm0.62 $  & $ 64.52\pm0.35 $  & $ 58.20\pm0.31 $  & $ 46.30\pm0.45 $  & $ 25.16\pm0.55 $ \\
& \textbf{NCE+MAE} & $ 68.75\pm0.54 $  & $ \boldsymbol{65.25\pm0.62} $  & $ \boldsymbol{59.22\pm0.36} $  & $ \boldsymbol{48.06\pm0.34} $  & $ \boldsymbol{25.50\pm0.76} $ \\
& \textbf{NCE+RCE} & $ 69.02\pm0.11 $  & $ \boldsymbol{65.31\pm0.07} $  & $ \boldsymbol{59.48\pm0.56} $  & $ \boldsymbol{47.12\pm0.62} $  & $ \boldsymbol{25.80\pm1.12} $ \\
\hline
\end{tabular}
\end{table*}

\begin{table*}[!t]
\caption{Test accuracy (\%) of different methods on benchmark datasets with clean or asymmetric label noise ($\eta \in [0.1, 0.4]$). The results (mean$\pm$std) are reported over 3 random runs and the top 2 best results are \textbf{boldfaced}.}
\label{tab:asym_robustness}
\centering
\small
\begin{tabular}{c|c|cccc}
\hline
\multirow{2}{*}{Datasets} & \multirow{2}{*}{Methods} & \multicolumn{4}{c}{Asymmetric Noise Rate ($\eta$)}\\
& &  0.1 & 0.2 & 0.3 & 0.4 \\ \hline \hline
\multirow{10}{*}{\begin{tabular}[c]{@{}c@{}}\\MNIST\\ \end{tabular}} 
 & CE & $ 98.53\pm0.11 $  & $ 96.75\pm0.31 $  & $ 92.98\pm1.41 $  & $ 85.74\pm2.70 $  \\
 & FL & $ 98.97\pm0.10 $  & $ 98.35\pm0.17 $  & $ 96.57\pm0.36 $  & $ 91.18\pm2.02 $ \\
 & GCE & $99.25\pm0.03 $  & $ 99.11\pm0.04 $  & $ 96.99\pm0.53 $  & $ 88.56\pm2.40 $ \\
 & NLNL & $ 98.38\pm0.17 $  & $ 95.98\pm0.58 $  & $ 91.52\pm1.14 $  & $ 86.36\pm0.40 $  \\
 & SCE & $ 99.15\pm0.07 $  & $ 99.05\pm0.05 $  & $ 97.96\pm0.40 $  & $ 91.89\pm3.32 $ \\
 \cline{2-6}
& \textbf{NFL+MAE} & $ 99.31\pm0.05 $  & $ 99.09\pm0.12 $  & $ 97.88\pm0.16 $  & $ \boldsymbol{93.52\pm0.19} $ \\
& \textbf{NFL+RCE} & $ \boldsymbol{99.33\pm0.06} $  & $ 99.13\pm0.01 $  & $ \boldsymbol{97.99\pm0.05} $  & $ \boldsymbol{93.59\pm0.82} $ \\
& \textbf{NCE+MAE} & $ 99.26\pm0.02 $  & $ \boldsymbol{99.21\pm0.04} $  & $ \boldsymbol{98.99\pm0.03} $  & $ 93.40\pm1.28 $ \\
& \textbf{NCE+RCE} & $ \boldsymbol{99.34\pm0.06} $  & $ \boldsymbol{99.17\pm0.02} $  & $ 97.94\pm0.21 $  & $ 93.12\pm1.17 $ \\
\hline
\multirow{10}{*}{\begin{tabular}[c]{@{}c@{}}\\CIFAR-10\\ \end{tabular}} 
 & CE & $ 87.38\pm0.16 $  & $ 83.62\pm0.15 $  & $ 79.38\pm0.28 $  & $ 75.00\pm0.50 $ \\
 & FL & $ 86.35\pm0.30 $  & $ 82.97\pm0.14 $  & $ 79.48\pm0.21 $  & $ 74.60\pm0.15 $ \\
 & GCE & $ 88.42\pm0.07 $  & $ 86.07\pm0.31 $  & $ 80.78\pm0.21 $  & $ 74.98\pm0.32 $ \\
 & NLNL & $ 88.54\pm0.25 $  & $ 84.74\pm0.08 $  & $ 81.26\pm0.43 $  & $ 76.97\pm0.52 $  \\
 & SCE & $ 88.13\pm0.21 $  & $ 83.92\pm0.07 $  & $ 79.70\pm0.27 $  & $ 78.20\pm0.03 $ \\
 \cline{2-6}
& \textbf{NFL+MAE} & $ 88.46\pm0.20 $  & $ 86.81\pm0.32 $  & $ 83.91\pm0.34 $  & $ 77.16\pm0.10 $\\
& \textbf{NFL+RCE} & $ \boldsymbol{90.20\pm0.15} $  & $ \boldsymbol{88.73\pm0.29} $  & $ \boldsymbol{85.74\pm0.22} $  & $ \boldsymbol{79.27\pm0.43} $\\
& \textbf{NCE+MAE} & $ 88.25\pm0.09 $  & $ 86.44\pm0.23 $  & $ 83.98\pm0.52 $  & $ 78.23\pm0.42 $  \\
& \textbf{NCE+RCE} & $ \boldsymbol{89.95\pm0.20} $  & $ \boldsymbol{88.56\pm0.17} $  & $ \boldsymbol{85.58\pm0.44} $  & $ \boldsymbol{79.59\pm0.40} $ \\
\hline \hline
\multirow{10}{*}{\begin{tabular}[c]{@{}c@{}}\\CIFAR-100\\ \end{tabular}} 
 & CE & $ 65.42\pm0.22 $  & $ 58.45\pm0.45 $  & $ 51.09\pm0.29 $  & $ 41.68\pm0.45 $  \\
 & FL & $ 64.79\pm0.18 $  & $ 58.59\pm0.81 $  & $ 51.26\pm0.18 $  & $ 42.15\pm0.44 $ \\
 & GCE & $ 61.98\pm0.81 $  & $ 59.99\pm0.83 $  & $ 53.99\pm0.29 $  & $ 41.49\pm0.79 $ \\
 & NLNL & $ 59.55\pm1.22 $  & $ 50.19\pm0.56 $  & $ 42.81\pm1.13 $  & $ 35.10\pm0.20 $ \\
 & SCE & $ 64.15\pm0.61 $  & $ 58.22\pm0.47 $  & $ 49.85\pm0.91 $  & $ 42.19\pm0.19 $ \\
 \cline{2-6}
& \textbf{NFL+MAE} & $ \boldsymbol{66.06\pm0.23} $  & $ \boldsymbol{63.10\pm0.22} $  & $ 56.19\pm0.61 $  & $ 43.51\pm0.42 $ \\
& \textbf{NFL+RCE} & $ \boldsymbol{66.13\pm0.31} $  & $ \boldsymbol{63.12\pm0.41} $  & $ 54.72\pm0.38 $  & $ 42.97\pm1.03 $ \\
& \textbf{NCE+MAE} & $ 65.71\pm0.34 $  & $ 62.38\pm0.60 $  & $ \boldsymbol{58.02\pm0.48} $  & $ \boldsymbol{47.22\pm0.30} $ \\
& \textbf{NCE+RCE} & $ 65.68\pm0.25 $  & $ 62.68\pm0.79 $  & $ \boldsymbol{57.82\pm0.41} $  & $ \boldsymbol{46.79\pm0.96} $ \\
\hline
\end{tabular}
\end{table*}

\subsection{Evaluation on Benchmark Datasets}\label{sec:benckmark_robust}

\noindent\textbf{Baselines.} We consider 3 state-of-the-art methods: 1) Generalized Cross Entropy (GCE) \cite{zhang2018generalized}; 2) Negative Learning for Noisy Labels (NLNL) \cite{kim2019nlnl}; and 3) Symmetric Cross Entropy (SCE) \cite{wang2019symmetric}. For APL, we consider 4 loss functions: 1) NCE+MAE, 2) NCE+RCE, 3) NFL+MAE and 4) NFL+RCE. We also train networks using CE and FL losses.

\noindent\textbf{Noise generation.}
The noisy labels are generated following standard approaches in previous works \cite{patrini2017making,ma2018dimensionality}.
Symmetric noise is generated by flipping labels in each class randomly to incorrect labels of other classes. For asymmetric noise, we flip the labels within a specific set of classes. For CIFAR-10, flipping TRUCK $\to$ AUTOMOBILE, BIRD $\to$ AIRPLANE, DEER $\to$ HORSE, CAT $\leftrightarrow$ DOG. For CIFAR-100, the 100 classes are grouped into 20 super-classes with each has 5 sub-classes, we then flip each class within the same super-class into the next in a circular fashion.
We vary the noise rate $\eta \in [0.2, 0.8]$ for symmetric noise, and $\eta \in [0.1, 0.4]$ for asymmetric noise.

\noindent\textbf{Networks and training.}
We use a 4-layer CNN for MNIST, an 8-layer CNN for CIFAR-10 and a ResNet-34 for CIFAR-100. We train the networks for 50, 120 and 200 epochs for MNIST, CIFAR-10, and CIFAR-100, respectively. For all the training, we use SGD optimizer with momentum 0.9 and cosine learning rate annealing. Weight decay is set to $1\times 10^{-3}$, $1\times 10^{-4}$ and $1\times 10^{-5}$ for MNIST, CIFAR-10 and CIFAR-100, respectively. The initial learning rate is set to 0.01 for MNIST/CIFAR-10 and 0.1 for CIFAR-100. Typical data augmentations including random width/height shift and horizontal flip are applied.

\noindent\textbf{Parameter setting.}
We tune the parameters for all baseline methods and find that the optimal settings match their original papers. Specifically, for GCE, we set $\rho=0.7$ (see detailed definition in Section \ref{sec:ngce}). For SCE, we set $A = -4$, and $\alpha = 0.01, \beta = 1.0$ for MNIST, $\alpha = 0.1, \beta = 1.0$ for CIFAR-10, $\alpha = 6.0, \beta = 0.1$ for CIFAR-100. For FL, we set $\gamma=0.5$. For our APL losses, we empirically set $\alpha = 1, \beta = 100$ for MNIST, $\alpha,\beta = 1$ for CIFAR-10, and $\alpha = 10, \beta = 0.1$ for CIFAR-100.

\noindent\textbf{Results.} The classification accuracies under symmetric label noise are reported in Table \ref{tab:sym_robustness}.
As can be seen, our APL loss functions achieved the top 2 best results in all test scenarios across all datasets. 
The superior performance of APL losses is more pronounced when the noise rates are extremely high and the dataset is more complex. For example, on CIFAR-10 with 0.6 symmetric noise, our APL losses NFL+RCE and NCE+RCE outperform the state-of-the-art robustness (72.85\% of NLNL) by more than 6\%.
On CIFAR-100 with 0.8 symmetric noise where CE and FL both fail to converge, our NCE+MAE and NCE+RCE outperform the state-of-the-art methods GCE and NLNL by at least 9\%. In several cases, all our 4 APL losses are better than baseline methods.

Results for asymmetric noise are reported in Table \ref{tab:asym_robustness}. Again, all top 2 best results are achieved by our APL loss functions across different datasets and noise rates. On CIFAR-100 with 0.4 asymmetric noise, the highest accuracy that can be achieved by current methods is 42.19\% (by SCE), which is still 5\% lower than our NCE+MAE and 4\% lower than our NCE+RCE. Comparing results in both Table \ref{tab:sym_robustness} and Table \ref{tab:asym_robustness}, we find that the best combination of our APL loss varies across different datasets, but within the same dataset, is quite consistent across different noise types and noise rates.

Overall, NCE+RCE demonstrates a consistently strong performance across different datasets. 
The strong performance of our APL losses verifies the importance of theoretically guaranteed robustness and ``Active+Passive" learning.
Our proposed APL framework can be used as a general principle for developing new robust loss functions.

\begin{table}[!ht]
\vspace{-0.1 in}
    \small
    \centering
    \caption{Test accuracy (\%) of APL losses NGCE+MAE and NGCE+RCE on CIFAR-10 under both symmetric and asymmetric noise. The top-2 best results are in \textbf{bold}.}
    \label{tab:ngce_plus}
    \begin{adjustbox}{width=0.49\textwidth}
    \begin{tabular}{l|cc|c}
    \hline
    \multirow{2}{*}{Methods}
    & \multicolumn{2}{c|}{Symmetric noise}  & \multicolumn{1}{c}{Asymmetric noise} \\
     & 0.4 & 0.8  & 0.4  \\
    \hline
    GCE & $ 83.33\pm0.39 $ & $ 29.08\pm0.80 $ & $ \boldsymbol{74.98\pm0.32} $ \\
    \textbf{NGCE+MAE} & $\boldsymbol{ 84.14\pm0.15 }$ & $\boldsymbol{ 50.55\pm1.08 } $ & $\boldsymbol{ 76.55\pm0.48 }$ \\
    \textbf{NGCE+RCE} & $\boldsymbol{ 85.76\pm0.26 }$ & $\boldsymbol{ 44.69\pm4.93 }$ & $ 71.65\pm0.68 $ \\
    \hline
    \end{tabular}
    \end{adjustbox}
\vspace{-0.1 in}
\end{table}

\subsection{Improving New Loss Functions using APL}\label{sec:ngce}
Next, we take GCE \cite{zhang2018generalized} as an example and show how to improve a new loss function using our APL framework.
Given a sample $\xx$, GCE loss is defined as: $GCE = \sum_{k=1}^{K} \qq(k|\xx)\frac{1-\pp(k|\xx)^\rho}{\rho}$,
where $\rho \in (0, 1]$. GCE reduces to the MAE/unhinged loss and CE loss when $\rho=1$ and $\rho \rightarrow 0$, respectively.
Following Eq. \eqref{eq:normalized_loss}, the Normalized Generalized Cross Entropy (NGCE) loss can be defined as:
$NGCE=(1-\pp(y|\xx)^{\rho})/(K - \sum_{k=1}^{K}\pp(k|\xx)^{\rho}).$

Both GCE and NGCE are active loss functions (eg. $\ell(f(\xx), k) = 0, \forall k \neq y$). Thus, following our APL in Eq. \eqref{eq:apl}, we can define two APL losses for NGCE: 1) NGCE+MAE and 2) NGCE+RCE. Here, we simply set $\alpha,\beta=1.0$ for both APL losses. We compare their performance to GCE (with $\rho=0.7$) on CIFAR-10 under both symmetric and asymmetric noise. As shown in Table \ref{tab:ngce_plus}, both NGCE+MAE and NGCE+RCE can improve the performance of GCE under different noise settings, except for NGCE+RCE under 0.4 asymmetric noise. Particularly, under 0.8 symmetric noise, NGCE+MAE is able to improve GCE by $>20\%$. A new loss function may have multiple terms, in this case, we can normalize its non-robust terms, and then add an active or passive loss into the loss function if there are missing.

\begin{table}[h]
\centering
\small
\caption{Top-1 validation accuracies (\%) on clean ILSVRC12 validation set of ResNet-50 models trained on WebVision using different loss functions, under the Mini setting in \cite{jiang2018mentornet}. The top-2 best results are in \textbf{bold}.}

\label{tab:webvision}
\begin{adjustbox}{width=0.49\textwidth}
\begin{tabular}{l|ccccc}
\hline
Loss & CE & GCE & SCE & \textbf{NCE+MAE} & \textbf{NCE+RCE}\\ \hline
Acc & 58.88 & 53.68 &  61.76 & $\bm{62.36}$ & $ \bm{62.64} $\\

\hline
\end{tabular}
\end{adjustbox}
\vspace{-0.1 in}
\end{table}

\subsection{Effectiveness on Real-world Noisy Labels}\label{sec:webvision}
Here, we test the effectiveness of our APL loss functions on large-scale real-world noisy dataset WebVision \cite{li2017webvision}. WebVision contains 2.4 million images of real-world noisy labels, crawled from the web (eg. Flickr and Google) based on the 1,000 class labels of ImageNet ILSVRC12 \cite{deng2009imagenet}. Here, we follow the ``Mini" setting in \cite{jiang2018mentornet} that only takes the first 50 classes of the Google resized image subset. We evaluate the trained networks on the same 50 classes of the ILSVRC12 validation set, which can be considered as a clean validation.
We compare our APL losses NCE+MAE and NCE+RCE with GCE and SCE. 
For each loss, we train a ResNet-50 \cite{he2016deep} using SGD for 250 epochs with initial learning rate 0.4, nesterov momentum 0.9 and weight decay $3\times 10^{-5}$ and batch size 512. The learning rate is multiplied by 0.97 after every epoch of training. 
We resize the images to $224 \times 224$. Typical data augmentations including random width/height shift, color jittering and random horizontal flip are applied. For GCE, we use the suggested $\alpha=0.7$, while for SCE, we use the setting with $A=-4, \alpha=10.0, \beta=1.0$. For our two APL losses, we set $\alpha=50.0, \beta=0.1$ for NCE+RCE and $\alpha=50.0, \beta=1.0$ for NCE+MAE.
The top-1 validation accuracies of different loss functions on the clean ILSVRC12 validation set (eg. only the first 50 classes) are reported in Table \ref{tab:webvision}. As can be observed, both our APL losses outperform existing loss functions GCE and SCE by a clear margin. This verifies the effectiveness of our APL against real-world label noise.

\section{Conclusions}\label{sec:conclusion}
In this paper, we revisited the robustness and sufficient learning properties of existing loss functions for deep learning with noisy labels. We revealed a new theoretical insight into robust loss functions that: \emph{a simple normalization can make any loss function robust to noisy labels}. Then, we highlighted that robustness alone is not enough for a loss function to train accurate DNNs, and existing robust loss functions all suffer from an underfitting problem. To address this problem, we characterize existing robust loss functions into ``Active" or ``Passive" losses, and then proposed a mutually boosted framework \emph{Active Passive Loss} (APL). APL allows us to create a family of new loss functions that not only have theoretically guaranteed robustness but also are effective for sufficient learning.
We empirically verified the excellent performance of our APL loss functions compared to state-of-the-art methods on benchmark datasets.
Our APL framework can serve as a basic principle for developing new robust loss functions.

\section*{Acknowledgements}
This research was undertaken using the LIEF HPC-GPU Facility hosted at the University of Melbourne with the assistance of LIEF Grant LE170100200.

\vfill\pagebreak
\bibliography{icml2020}
\bibliographystyle{icml2020}


\newpage
\appendix
\onecolumn

\setcounter{lemma}{0}

\section{Proofs for Lemma \ref{lemma_1}, Lemma \ref{lemma_2} and Lemma \ref{lemma_3}}\label{appendix_proof}
Our proofs are inspired by \cite{ghosh2017robust}.

\begin{lemma}
In a multi-class classification problem, any normalized loss function $\L_{\text{norm}}$ is noise tolerant under symmetric (or uniform) label noise, if noise rate $\eta < \frac{K-1}{K}$.
\end{lemma}
\begin{proof}
For symmetric label noise, the noise risk can be defined as:
\begin{align*}
	\small
	R^\eta(f) & =  \E_{\xx, \hat{y}} \L_{norm}(f(\xx), \hat{y}) =  \E_{\xx} \E_{y | \xx} \E_{\hat{y} | \xx, y} \L_{norm}(f(\xx), \hat{y}) \\
		&= \E_{\xx} \E_{y | \xx} \Big[ (1-\eta) \L_{norm}(f(\xx), y) + \frac{\eta}{K-1} \sum_{k\neq y} \L_{norm}(f(\xx), k) \Big] \\
		& =  (1 - \eta) R(f) +  \frac{\eta}{K-1} \bigg(\E_{\xx, y}\bigg[\sum_{k=1}^{K}\L_{norm}(f(\xx), k)\bigg] - R(f)\bigg)\\
		& = R(f)\left(1-\frac{\eta K}{K-1}\right) + \frac{\eta}{K-1},
\end{align*}
where the last equality holds due to $\sum_{k=1}^{K}\L_{norm}(f(\xx), k) = 1$, following Eq. \eqref{eq:normalized_loss}. Thus,
	\[R^\eta(f^*)-R^\eta(f)=(1-\frac{\eta K}{K-1})(R(f^*)-R(f)) \leq 0,\]
because $\eta < \frac{K-1}{K}$ and $f^*$ is a global minimizer of $R(f)$. This proves $f^*$ is also the global minimizer of risk $R^\eta(f)$, that is, $\L_{norm}$ is noise tolerant to symmetric label noise. 
\end{proof}

\begin{lemma}
In a multi-class classification problem, given $R(f^{*})=0$ and $0 \leq \L_{\text{norm}}(f^{*}(\xx), k) \leq \frac{1}{K-1}$, any normalized loss function $\L_{\text{norm}}$ is noise tolerant under asymmetric (or class-conditional) label noise, if noise rate $\eta_{jk} < 1- \eta_y$.
\end{lemma}
\begin{proof}
For asymmetric or class-conditional noise, $1-\eta_y$ is the probability of a label being correct (\textit{i.e.,} $k=y$), and the noise condition $\eta_{y k} < 1-\eta_{y}$ generally states that a sample $\xx$ still has the highest probability of being in the correct class $y$, though it has probability of $\eta_{y k}$ being in an arbitrary noisy (incorrect) class $k \neq y$.
Considering the noise transition matrix between classes $[\eta_{i j}], \forall i,j \in \{1,2, \cdots, K\}$, this condition only requires that the matrix is diagonal dominated by $\eta_{i i}$ (\textit{i.e.,} the correct class probability $1 - \eta_{y}$). Following the symmetric case, here we have,
	\small{
	\begin{align}
	\label{eq:cc_1}
	\begin{split}
	R^\eta(f) & = \E_{\xx, \hat{y}} \L_{norm}(f(\xx), \hat{y}) =  \E_{\xx} \E_{y | \xx} \E_{\hat{y} | \xx, y} \L_{norm}(f(\xx), \hat{y}) \\ 
	& = \E_{\xx} \E_{y | \xx} \Big[ (1- \eta_{y}) \L_{norm}(f(\xx), y) +   \sum_{k \neq y} \eta_{y k} \L_{norm} (f(\xx), k) \Big] \\
	& = \E_{\xx, y} \Big[ (1-\eta_{y})\Big(\sum_{k=1}^{K}\L_{norm}(f(\xx), k) - \sum_{k \neq y} \L_{norm}(f(\xx), k)\Big)\Big] + \E_{\xx, y} \Big[\sum_{k \neq y} \eta_{y k} \L_{norm}(f(\xx), k)\Big] \\
	& = \E_{\xx, y} \Big[ (1-\eta_{y})\big(1 - \sum_{k \neq y} \L_{norm}(f(\xx), k)\big)\Big] + \E_{\xx, y}\Big[ \sum_{k \neq y} \eta_{y k} \L_{norm}(f(\xx), k)\Big] \\
	& = \E_{\xx, y} (1-\eta_{y})-\E_{\xx, y} \Big[\sum_{k \neq y}(1-\eta_{y}-\eta_{y k}) \L_{norm}(f(\xx), k)\Big].
	\end{split}
	\end{align}	}
	As $f^{\ast}_{\eta}$ is the minimizer of $R^\eta(f)$, $R^{\eta}(f_{\eta}^{\ast})-R^{\eta}(f^{\ast}) \leq  0$. So, from \ref{eq:cc_1} above, we have,
	\begin{align}
	\label{eq:cc_2}
	\E_{\xx, y}\Big[\sum_{k\neq y}(1-\eta_{y}-\eta_{y k})\big(\underbrace{\L_{norm} (f^{\ast}(\xx),k)}_{\L_{norm}^{*}}-\underbrace{\L_{norm}(f^{\ast}_{\eta}(\xx),k)}_{\L_{norm}^{\eta *}}\big)\Big] \leq 0.
	\end{align}
	Next, we prove, $f^{\ast}_{\eta}=f^{\ast}$ holds following Eq. \eqref{eq:cc_2}. First, $(1-\eta_{y}-\eta_{y k})>0$ as per the assumption that $\eta_{y k} < 1-\eta_{y}$. Thus, $\L_{norm}^{*} - \L_{norm}^{\eta *} \leq 0$ for Eq. \eqref{eq:cc_2} to hold.
	Since we are given $R(f^*)=0$, we have $\L(f^*(\xx), y) = 0$. 
	Thus, following the definition of $\L_{norm}$ in Eq. \eqref{eq:normalized_loss} and assumption $\L_{\text{norm}}(f^{*}(\xx), k) \leq \frac{1}{K-1}$, we have $\L_{norm}(f^*(\xx),k)= \frac{\L(f^*(\xx)=0, k)}{\sum_{j}^{K}\L(f^*(\xx), j)} = \frac{1}{K-1}$, for all $k\neq y$. Also, we have $\L_{norm}(f_\eta^*(\xx) , k) =\frac{\L(f_\eta^*(\xx),k)}{\sum_{j}^{K}\L(f_\eta^*(\xx), j)} \leq \frac{1}{K-1}$, $\forall k\neq y$. Thus, for Eq. \eqref{eq:cc_2} to hold (\textit{e.g.} $\L_{norm}(f_\eta^*(\xx),k) \geq \L_{norm}(f^*(\xx),k)$), it must be the case that $p_k=0,\;\forall k\neq y$, that is, $\L_{norm}(f_\eta^*(\xx),k) = \L_{norm}(f^*(\xx),k)$ for all $k \in \{1,2, \cdots, K\}$, thus $f^{\ast}_{\eta}=f^{\ast}$ which completes the proof.
\end{proof}

\begin{lemma}
$\forall \alpha, \forall \beta$, if $\L_{\text{Active}}$ and $\L_{\text{Passive}}$ are noise tolerant, then $\L_{\text{APL}} = \alpha \cdot \L_{\text{Active}} + \beta \cdot \L_{\text{Passive}}$ is noise tolerant.
\end{lemma}
\begin{proof}
Let $\alpha, \beta \in \mathbb{R}$, then
$\sum_{j}^{K} \L_{\text{APL}}(f(\xx), j) =  \alpha \cdot \sum_{j}^{K} \L_{\text{Active}}(f(\xx), j) + \beta \cdot \sum_{j}^{K} \L_{\text{Passive}}(f(\xx), j) = \alpha \cdot C_{\text{Active}} + \beta \cdot C_{\text{Passive}} = C$.
Following our proof for Lemma \ref{lemma_1}, for symmetric noise, we have,
\begin{align*}
	\small
	R^\eta(f) & = R(f)\left(1-\frac{\eta K}{K-1}\right) + \frac{(\alpha \cdot C_{\text{Active}} + \beta \cdot C_{\text{Passive}})\eta}{K-1}.
\end{align*}
Thus, $R^\eta(f^*)-R^\eta(f)=(1-\frac{\eta K}{K-1})(R(f^*)-R(f)) \leq 0$.
Given $\eta < \frac{K-1}{K}$ and $f^*$ is a global minimizer of $R(f)$, $R(f^*)-R(f)$, that is, $f^*$ is also the global minimizer of risk $R^\eta(f)$. Thus, $\L_{\text{APL}}$ is noise tolerant to symmetric label noise. 

Following our proof for Lemma \ref{lemma_2}, for asymmetric noise, we have,

\small{
	\begin{align}
	\label{eq:cc_3}
	\begin{split}
	R^\eta(f) & = (\alpha \cdot C_{\text{Active}} + \beta \cdot C_{\text{Passive}}) \E_{\xx, y} (1-\eta_{y})-\E_{\xx, y} \Big[\sum_{k \neq y}(1-\eta_{y}-\eta_{y k}) \L_{norm}(f(\xx), k)\Big].
	\end{split}
	\end{align}	}
	
As $f^{\ast}_{\eta}$ is the minimizer of $R^\eta(f)$, $R^{\eta}(f_{\eta}^{\ast})-R^{\eta}(f^{\ast}) \leq  0$. So, from \ref{eq:cc_3} above, we can derive the same equation as Eq. \eqref{eq:cc_2},
	\begin{align}
	\label{eq:cc_4}
	\E_{\xx, y}\Big[\sum_{k\neq y}(1-\eta_{y}-\eta_{y k})\big(\underbrace{\L_{\text{APL}} (f^{\ast}(\xx),k)}_{\L_{\text{APL}}^{*}}-\underbrace{\L_{\text{APL}}(f^{\ast}_{\eta}(\xx),k)}_{\L_{\text{APL}}^{\eta *}}\big)\Big] \leq 0.
	\end{align}
 Thus, we can follow the same proof from Eq. \eqref{eq:cc_2}, to $f^{*}_{\eta} = f^{*}$, that is, $\L_{\text{APL}}$ is also noise tolerant to asymmetric noise.
\end{proof}




\end{document}